\documentclass[a4paper,fleqn]{cas-sc}

\usepackage[numbers,sort&compress]{natbib}

\def\tsc#1{\csdef{#1}{\textsc{\lowercase{#1}}\xspace}}
\tsc{WGM}
\tsc{QE}
\usepackage{amsmath,amssymb}      
\usepackage{xcolor}       
\usepackage{enumitem}     
\usepackage[most]{tcolorbox}  
\usepackage[table]{xcolor}
\usepackage{booktabs}
\usepackage{multirow}
\usepackage{makecell}
\usepackage{graphicx}
\usepackage{tabularx}
\usepackage{array}
\usepackage{algorithm}
\usepackage{algorithmicx}
\usepackage{algpseudocode}

\usepackage{mdframed}

\newmdenv[
    linewidth=0.7pt,
    linecolor=black!35,
    backgroundcolor=black!1,
    roundcorner=2pt,
    skipabove=3pt,
    skipbelow=3pt,
    innerleftmargin=7pt,
    innerrightmargin=7pt,
    innertopmargin=5pt,
    innerbottommargin=5pt
]{casebox}

\newcommand{\casetitle}[1]{%
    \par\vspace{3pt}
    \noindent\textbf{#1}\par
    \vspace{1.5pt}
    \hrule
    \vspace{3pt}
}
\newtheorem{theorem}{Theorem}

\newdefinition{rmk}{Remark}
\newproof{pf}{Proof}
\newproof{pot}{Proof of Theorem \ref{thm}}
\newtheorem{proposition}[theorem]{Proposition}

\begin{document}
\let\WriteBookmarks\relax

\makeatletter
\renewcommand\subsection{\@startsection{subsection}{2}{\z@}%
    {10pt \@plus 3\p@ \@minus 2\p@}%
    {4pt}%
    {\ssectionfont\raggedright}}
\makeatother

\renewcommand{\topfraction}{0.98}
\renewcommand{\textfraction}{0.01}
\renewcommand{\floatpagefraction}{0.8}
\setcounter{topnumber}{4}
\setcounter{totalnumber}{5}

\shorttitle{Hexagonal Semantic Reasoning for LLMs}  

\shortauthors{Author et al.}

\title [mode = title]{Semiotic logical hexagon theory for LLM logical reasoning}

\tnotemark[1] 


\tnotetext[1]{This paper is a substantially revised and extended version of a paper which first appeared at the 2026 Annual Meeting of the Association for Computational Linguistics (ACL 2026 Oral), under the title \emph{Semantic-Aware Logical Reasoning via a Semiotic Framework}.}

\shortauthors{Zhang et al.}  

%

\author[1]{Yunyao Zhang}[orcid=0009-0001-3412-9262]
\fnmark[1]
\ead{ikostar@hust.edu.cn}
\credit{Writing -- review \& editing, Writing -- original draft, Methodology, Conceptualization}

\author[1]{Xinglang Zhang}
\fnmark[1]
\ead{normanspark@hust.edu.cn}
\credit{Investigation, Validation, Writing -- review \& editing}

\author[1]{Zeliang Chen}
\credit{Writing -- review \& editing}

\author[1]{Junqing Yu}
\credit{Writing -- review \& editing}

\author[1]{Zikai Song}[orcid=0009-0006-6651-2027]
\cormark[1]
\ead{skyesong@hust.edu.cn}
\credit{Writing -- review \& editing, Supervision}

\affiliation[1]{organization={Huazhong University of Science and Technology},
            addressline={},
            city={Wuhan},
            postcode={430074},
            state={Hubei},
            country={China}}

\cortext[1]{Corresponding author}

\fntext[1]{Yunyao Zhang and Xinglang Zhang contributed equally to this work.}




\begin{abstract}
Large language models (LLMs) have become powerful tools for language understanding and logical reasoning. However, they still make mistakes when a problem requires both understanding meaning and following logic. A key reason is that natural-language statements often carry implicit semantic relations before any formal reasoning begins. If these hidden meanings are not properly organized, the model may reach incorrect conclusions even when the subsequent reasoning process appears logically valid. Existing methods improve reasoning through decomposition, symbolic translation, external solvers, or self-verification, but pay comparatively less attention to the semantic structure on which reasoning depends.
In this paper, we further investigate how semantic organization influences logical reasoning in LLMs. To this end, we propose \textbf{HexLogicAgent}, a framework that first organizes the meaning of natural-language statements and then guides logical reasoning through structured verification.
In our investigation, we also make two observations. First, incomplete semantic representations, rather than deductive inference itself, are a major source of logical reasoning failures in LLMs. Second, explicitly modeling the complete structure of semantic opposition substantially delays the degradation of reasoning performance as logical complexity increases.
Experiments on challenging logical reasoning benchmarks demonstrate that HexLogicAgent consistently improves reasoning reliability across multiple LLMs. The core idea is supported by a \textbf{logical hexagon theory}, which explains why a complete structure of opposing meanings is necessary for reliable reasoning.

\end{abstract}





\maketitle

\section{Introduction}

The rapid development of Large Language Models (LLMs), including DeepSeek~\citep{DeepseekR1-2025deepseek}, ChatGPT~\cite{Gpt-4o-2024gpt}, and Qwen~\citep{qwen2.5-2025qwen2}, has brought substantial progress to natural-language understanding and reasoning. These models can help users read complex texts~\citep{TnT-LLM-KDD2024}, connect scattered information~\citep{LLMsConnecting-NIPS2024}, generate explanations~\citep{LLMExplanations-ACL2024}, and solve problems across a wide range of domains~\citep{LLMDomain-NEURIPS2023}. By doing so, they improve the efficiency of knowledge work and support tasks that require careful thinking, such as commonsense reasoning~\citep{Commonsense-reasoning-2024-candle}, mathematical problem solving~\citep{Mathematical-reasoning-2023mathcoder,Geometric-reasoning-2024deep}, philosophical thinking~\citep{philosophical-thinking-2013critical}, and social analysis~\citep{zhang2026couplingmacrodynamicsmicro,zhang2026intervensiminterventionawaresocialnetwork,wang2026seeingwiderjointspatiotemporal,GAS3-2025}. Moreover, LLMs are increasingly used as reasoning assistants in domains where conclusions must be drawn from natural-language statements, contextual premises, and incomplete evidence.

\begin{figure}[width=\textwidth,pos=h!]
    \centering
    \includegraphics[
        width=0.95\textwidth,
        height=0.4\textheight,
        keepaspectratio
    ]{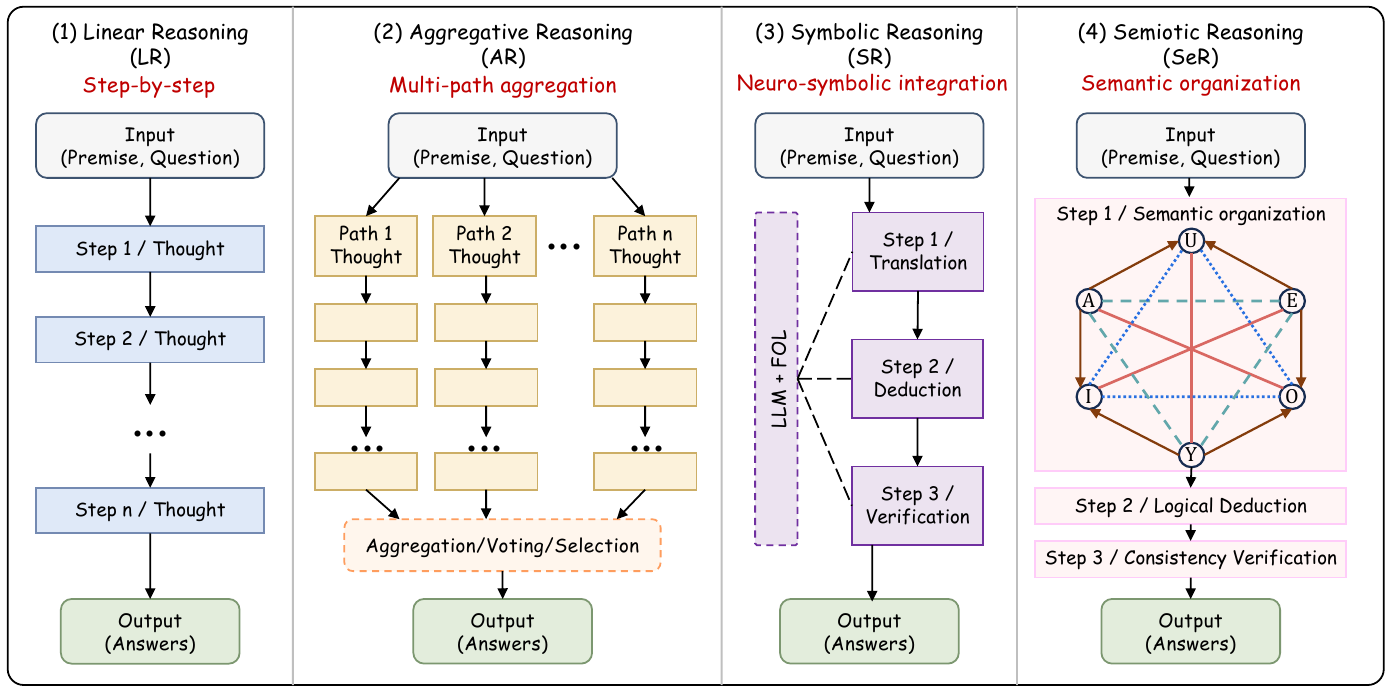}
    \caption{
    \textbf{Comparison of representative LLM reasoning paradigms.}
    \textbf{(1) Linear Reasoning (LR)} follows a single reasoning chain,
    \textbf{(2) Aggregative Reasoning (AR)} aggregates multiple paths, and
    \textbf{(3) Symbolic Reasoning (SR)} couples LLMs with symbolic modules.
    \textbf{(4) Semiotic Reasoning (SeR)}, instantiated by \textbf{HexLogicAgent},
    organizes proposition meanings through semiotic structures before deduction and verification.
    }
    \label{fig:teaser}
\end{figure}

While LLMs offer significant benefits for reasoning, they also raise important reliability concerns~\citep{Reliable-reasoning-ASC2026}. A reasoning problem in natural language is not only a chain of logical steps~\citep{Linguistic-complexity1998linguistic}, but also a set of statements whose meanings must be properly understood before deduction begins~\citep{Ambiguity-1993ambiguity}. This challenge comes from two coupled sources: (1) \textit{semantic complexity}, where statements may be abstract, context-dependent, only partly specified, or organized by different forms of opposition; and (2) \textit{logical complexity}, where conclusions must be derived from contextual premises through valid reasoning steps. When these meanings are not handled carefully, an LLM may treat different kinds of opposition as the same kind of conflict, or may turn a conditional statement into an unconditional rule. As a result, the final answer may look logically well formed while still being semantically wrong~\citep{LLMFormalStable-USTC-WWW2026}. Therefore, our research aims to improve LLM reasoning by organizing the meaning of propositions before applying logical deduction.

Therefore, connecting semantic interpretation with logical deduction is crucial for reliable LLM reasoning. This requires a structured account of how propositions relate to one another before deduction begins. Without such semantic organization, later reasoning steps may be built on an incomplete or mistaken interpretation of the original problem.

This study focuses on a typical LLM reasoning setting, where a model receives natural-language premises and a question, and produces a conclusion with or without external reasoning tools~\citep{formal-logic-2003introduction}. Existing methods can be grouped into three categories:
(1) \textbf{Linear Reasoning (LR)} methods, such as Naive Prompting and Chain-of-Thought (CoT) prompting~\citep{COT-2022chain}, which elicit step-by-step reasoning;
(2) \textbf{Aggregative Reasoning (AR)} approaches~\citep{CLOVER-ICLR-2024divide, Cumulative-reasoning-2023-TsinghuIIIS, TOT-2023tree, Determlr-ACL-2024-RenDaGaoling, multi-step-2024exploring}, which combine multiple reasoning paths; and
(3) \textbf{Symbolic Reasoning (SR)} frameworks~\citep{NL2FOL-2023harnessing, Baseline-Aristotle2024xu, Symbolic-COT-2024faithful, hahn2022formal, LogicLM-2023}, which couple LLMs with formal translation, symbolic deduction, or verification.
Although effective, these methods mainly improve the reasoning process after propositions have been interpreted. They give limited attention to how proposition meanings should be organized before deduction, often assuming clean predicates, stable meanings, and unambiguous contexts. As a result, abstraction, conflicting stances, existential commitments, and partial compatibility remain underrepresented.

This gap motivates our central question: is there a complete semantic structure that can organize proposition meanings before deduction and support more reliable LLM reasoning? To answer this question, we build on the study of semantic opposition in semiotics and logic. Greimas' semiotic square~\citep{Greimas1982semiotics, Greimas-meaning1987} offers a classical starting point by organizing meaning through structured oppositions. Our previous LogicAgent framework~\citep{zhang2026semanticawarelogicalreasoningsemiotic} follows this idea and uses the semiotic square to guide multi-perspective reasoning. However, its four-part structure still leaves some proposition relations underrepresented. In this work, our answer is given in Proposition~\ref{prop:hexagonal-completion}, which shows that the logical hexagon~\citep{hexagon-2012,logical-hexagon-2012} provides a more complete structure for semantic-aware logical reasoning~\citep{square2hexagon-2012}.

Our main contributions are as follows:
\begin{itemize}
\item We generalize LogicAgent from the semiotic square to the logical hexagon, enabling complete semantic organization before logical deduction. This extension improves semantic interpretation from $13\%$ to $100\%$ on representative semantic understanding cases and translates into consistent gains of $+2.66$, $+2.74$, and $+2.41$ average accuracy points across three frontier LLM backbones.

\item We formalize the logical hexagon in the context of semantic-aware LLM reasoning through formal definitions of semantic relations, semantic completeness, and reasoning reliability. Our theoretical and empirical analyses further reveal that incomplete semantic representations, rather than deductive inference itself, are a major source of logical reasoning failures in LLMs, while explicitly modeling complete semantic opposition substantially delays reasoning degradation under increasing logical complexity.

\end{itemize}

This work is an extended version of our ACL 2026 main conference paper, selected for an oral presentation~\citep{zhang2026semanticawarelogicalreasoningsemiotic}. Compared with the conference version, this article provides a formalization of the logical hexagon, formal analysis, full implementation details, and additional experiments. The rest of the paper is organized as follows. Section~\ref{sec:preliminaries} introduces the preliminaries and problem formulation. Section~\ref{sec:Method} presents the methodology, including the logical hexagon and HexLogicAgent. Section~\ref{sec:experiments_analysis} presents the experiments and analysis. Section~\ref{sec:related_work} reviews related work. Finally, Section~\ref{sec:conclusion} concludes the paper and outlines future research directions.

\section{Preliminaries}
\label{sec:preliminaries}

We formalize semiotic reasoning in first-order logic (FOL) (§~\ref{sec:fol-prelim}), instantiate Greimas' semiotic square under FOL semantics (§~\ref{sec:greimas-square}), and show that the four-position structure underrepresents existential and compatibility relations (Prop.~\ref{prop:square-limitation}). This motivates the logical hexagon (§~\ref{sec:logical_hexagon}) as a more complete structure for organizing proposition meanings.

\subsection{First-Order Logic}
\label{sec:fol-prelim}

We use FOL to represent the logical form of natural-language (NL) propositions. Let $x$ denote an individual variable, $M(x)$ a subject predicate, and $W(x)$ a property predicate. Universal and existential propositions are written as
\begin{equation}
\forall x \bigl(M(x)\rightarrow W(x)\bigr),
\qquad
\exists x \bigl(M(x)\land W(x)\bigr).
\end{equation}
The first formula states that every object satisfying $M$ also satisfies $W$, while the second states that at least one object satisfies both $M$ and $W$. Since universal implications may become vacuously true when the subject class is empty, we explicitly assume existential import, i.e., $\exists x\,M(x)$, whenever opposition relations are analyzed.

\begin{figure}[width=\textwidth,pos=h!]
    \centering
    \includegraphics[
        width=0.9\textwidth,
        height=0.3\textheight,
        keepaspectratio
    ]{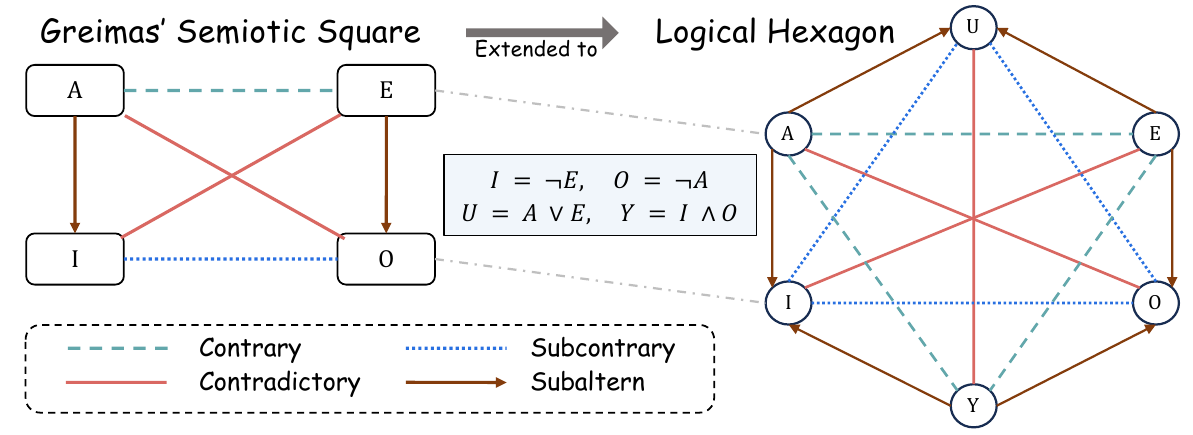}
    \caption{
    \textbf{Comparison between the semiotic square and the logical hexagon.}
    The semiotic square captures contrary, contradictory, and subaltern relations, while the logical hexagon extends it with existential and mixed-state positions. This enables HexLogicAgent to model subcontrary relations, distinguish local counterexamples from global negation, and represent partial proposition satisfaction.
    }
    \label{fig:square-to-hexagon}
\end{figure}

\subsection{Greimas' Semiotic Square and Its Limitation}
\label{sec:greimas-square}

Natural-language reasoning often involves more than binary truth evaluation.
A proposition may have a direct denial, a contrary stance, and weaker or stronger variants.
Greimas' semiotic square~\citep{Greimas1982semiotics,Greimas-meaning1987} provides a classical way to organize such semantic relations.
Starting from a pair of contrary terms, the square introduces the negation of each term and forms a four-position structure.
This structure organizes proposition meanings through three basic relations:
(1) \textit{contrariety}, where two terms cannot both hold but may both fail to hold;
(2) \textit{contradiction}, where one term directly negates the other; and
(3) \textit{implication} or \textit{complementarity}, where one position is semantically connected to another through a dependency relation.
To connect this structure with FOL, consider the following four propositions:
\begin{align}
A &:= \forall x \bigl(M(x) \rightarrow W(x)\bigr), 
& O \equiv \lnot A &:= \exists x \bigl(M(x) \land \lnot W(x)\bigr), \\
E &:= \forall x \bigl(M(x) \rightarrow \lnot W(x)\bigr), 
& I \equiv \lnot E &:= \exists x \bigl(M(x) \land W(x)\bigr).
\end{align}
In this instantiation, $A$ and $E$ serve as the initial contrary terms: 
$A$ states that all $M$-objects satisfy $W$, whereas $E$ states that all $M$-objects satisfy $\lnot W$. 
The other two positions are generated by negating these contrary terms: $O$ is the contradiction of $A$, and $I$ is the contradiction of $E$. 
Contrariety holds between $A$ and $E$ when the subject class is non-empty, i.e., when $\exists x\,M(x)$ holds, since the two universal claims cannot both be true in that case but can both be false. 
The implication or complementarity relation is existentially conditioned: with existential import, $A$ supports $I$, and $E$ supports $O$.

This FOL instantiation shows both the usefulness and the limitation of the four-position square.
It organizes universal contraries and their contradictory negations, but it does not make existential commitment itself an explicit semantic position.
Nor does it represent composite meanings such as the coexistence of $I$ and $O$, which is needed to distinguish universal uniformity from partial compatibility.

\begin{proposition}[Existential Limitation and Incompleteness of Greimas' Square under FOL Semantics]
\label{prop:square-limitation}
Let a \\ square-based structure be instantiated by the four propositions
$A,E,I,O$ defined above. Then:
\begin{enumerate}
    \item the existential propositions $I$ and $O$ do not form a contrary pair under ordinary FOL semantics; they can be jointly true, and they can be jointly false only when the subject class is empty;
    \item the four-position structure is incomplete for semiotic reasoning, because it does not represent the composite meanings $A \lor E$ and $I \land O$ as explicit semantic positions.
\end{enumerate}
\end{proposition}

\begin{pf}
\textbf{Claim 1.}
$I$ and $O$ are not contraries because they can be jointly true. 
For instance, in a model with two objects $a,b$ such that 
$M(a)\land W(a)$ and $M(b)\land \lnot W(b)$, both
$I=\exists x(M(x)\land W(x))$ and 
$O=\exists x(M(x)\land \lnot W(x))$ hold. 
Since contraries cannot both be true, $I$ and $O$ do not form a contrary pair.

Their joint falsity arises only from an empty subject class:
\begin{align}
\lnot I \land \lnot O
&\equiv
\forall x(M(x)\rightarrow \lnot W(x))
\land
\forall x(M(x)\rightarrow W(x))  \notag\\
&\equiv
\forall x\,\lnot M(x).
\end{align}
Thus, making $I$ and $O$ both false depends on the absence of $M$-objects, not on a genuine contrary relation between the two existential propositions.

\textbf{Claim 2.}
The composite meanings $I\land O$ and $A\lor E$ are not reducible to any single node in $A,E,I,O$.
The formula
\[
I\land O
\equiv
\exists x(M(x)\land W(x))
\land
\exists y(M(y)\land \lnot W(y))
\]
expresses a mixed subject class. 
It is stronger than either $I$ or $O$ alone and is incompatible with both $A$ and $E$.

Conversely,
\[
A\lor E
\equiv
\forall x(M(x)\rightarrow W(x))
\lor
\forall x(M(x)\rightarrow \lnot W(x))
\]
expresses a uniform subject class. 
It is weaker than either $A$ or $E$ alone, but it excludes mixed cases where both $I$ and $O$ hold. 
Hence neither $I\land O$ nor $A\lor E$ is equivalent to any single proposition among $A,E,I,O$.
Therefore, the four-position square does not encode these composite meanings as explicit semantic positions.
\end{pf}

\subsection{Logical Hexagon Theory}
\label{sec:logical_hexagon}

The logical hexagon extends the four square positions $A,E,I,O$ by adding two composite positions:
\begin{align}
Y &:= I \land O
   \equiv
   \exists x\bigl(M(x)\land W(x)\bigr)
   \land
   \exists y\bigl(M(y)\land \lnot W(y)\bigr), \\
U &:= A \lor E
   \equiv
   \forall x\bigl(M(x)\rightarrow W(x)\bigr)
   \lor
   \forall x\bigl(M(x)\rightarrow \lnot W(x)\bigr).
\end{align}
Here, $Y$ represents a mixed subject class, where both existential alternatives are realized, while $U$ represents a uniform subject class, where one of the two universal alternatives holds.
Under the normalized subject--property tuple $\langle x,M,W\rangle$, the six positions ${A,E,I,O,U,Y}$ form a complete hexagonal closure: $A$ and $E$ capture universal alternatives, $I$ and $O$ capture existential alternatives, and $U,Y$ make the uniform and mixed compositions explicit as first-class semantic positions.

\begin{proposition}[Hexagonal Completion under FOL Semantics]
\label{prop:hexagonal-completion}
Let $A,E,I,O$ be the four propositions defined above, and assume existential import over the subject class, i.e., $\exists x\,M(x)$.
Together with $Y:=I\land O$ and $U:=A\lor E$, the six propositions $\{A,E,I,O,U,Y\}$ form a logical hexagon with the following four relations:
\begin{enumerate}
    \item \textit{Contrariety:} $A,E,Y$ are pairwise contraries;
    \item \textit{Subcontrariety:} $U,I,O$ are pairwise subcontraries;
    \item \textit{Contradiction:} $A$ contradicts $O$, $E$ contradicts $I$, and $U$ contradicts $Y$;
    \item \textit{Subalternation:} $A\Rightarrow I$, $E\Rightarrow O$, $A\Rightarrow U$, $E\Rightarrow U$, $Y\Rightarrow I$, and $Y\Rightarrow O$.
\end{enumerate}
\end{proposition}

\begin{pf}
\textit{(1) Contrariety.}
The propositions $A$ and $E$ cannot both be true under existential import, since a non-empty subject class cannot be uniformly $W$ and uniformly $\lnot W$ at the same time.
The proposition $A$ is also incompatible with $Y$, because $Y=I\land O$ entails $O$, and $O\equiv\lnot A$.
Similarly, $E$ is incompatible with $Y$, because $Y=I\land O$ entails $I$, and $I\equiv\lnot E$.
Thus, $A,E,Y$ are pairwise contraries.

\textit{(2) Subcontrariety.}
Under existential import and classical bivalence, $I$ and $O$ cannot both be false, because every existing $M$-object must satisfy either $W$ or $\lnot W$.
Moreover, if $U$ is false, then
\[
\lnot U
\equiv
\lnot(A\lor E)
\equiv
\lnot A\land \lnot E
\equiv
O\land I
\equiv
Y.
\]
Hence $U$ cannot be false together with either $I$ or $O$, since $\lnot U$ entails both $I$ and $O$.
Therefore, $U,I,O$ are pairwise subcontraries.

\textit{(3) Contradiction.}
By standard FOL equivalence, $\lnot A\equiv O$ and $\lnot E\equiv I$.
Hence $A$ contradicts $O$, and $E$ contradicts $I$.
From the equivalence above, $\lnot U\equiv Y$, so $U$ contradicts $Y$.

\textit{(4) Subalternation.}
The implications follow directly from the definitions.
Under existential import, $A\Rightarrow I$ and $E\Rightarrow O$.
Since $U=A\lor E$, we have $A\Rightarrow U$ and $E\Rightarrow U$.
Since $Y=I\land O$, we have $Y\Rightarrow I$ and $Y\Rightarrow O$.
\end{pf}

\section{Methodology}
\label{sec:Method}

We extend the original LogicAgent framework into a hexagon-guided reasoning framework, termed \textbf{Hexagon-guided LogicAgent} (\textbf{HexLogicAgent}). 
HexLogicAgent keeps the three-stage design of LogicAgent, but replaces the four-position semiotic square with a six-position logical hexagon.
As shown in Figure~\ref{fig:square-to-hexagon}, HexLogicAgent consists of three stages:
(1) \textit{Hexagonal Semantic Structuring},
(2) \textit{Logical Reasoning}, and
(3) \textit{Hexagon-guided Reflective Verification}.
The first stage organizes a target proposition into six semantically related positions, the second stage performs deduction for each relevant position, and the third stage verifies the resulting judgments according to the opposition relations encoded in the logical hexagon.

\begin{figure}[width=\textwidth,pos=h!]
    \centering
    \includegraphics[
        width=0.95\textwidth,
        height=0.4\textheight,
        keepaspectratio
    ]{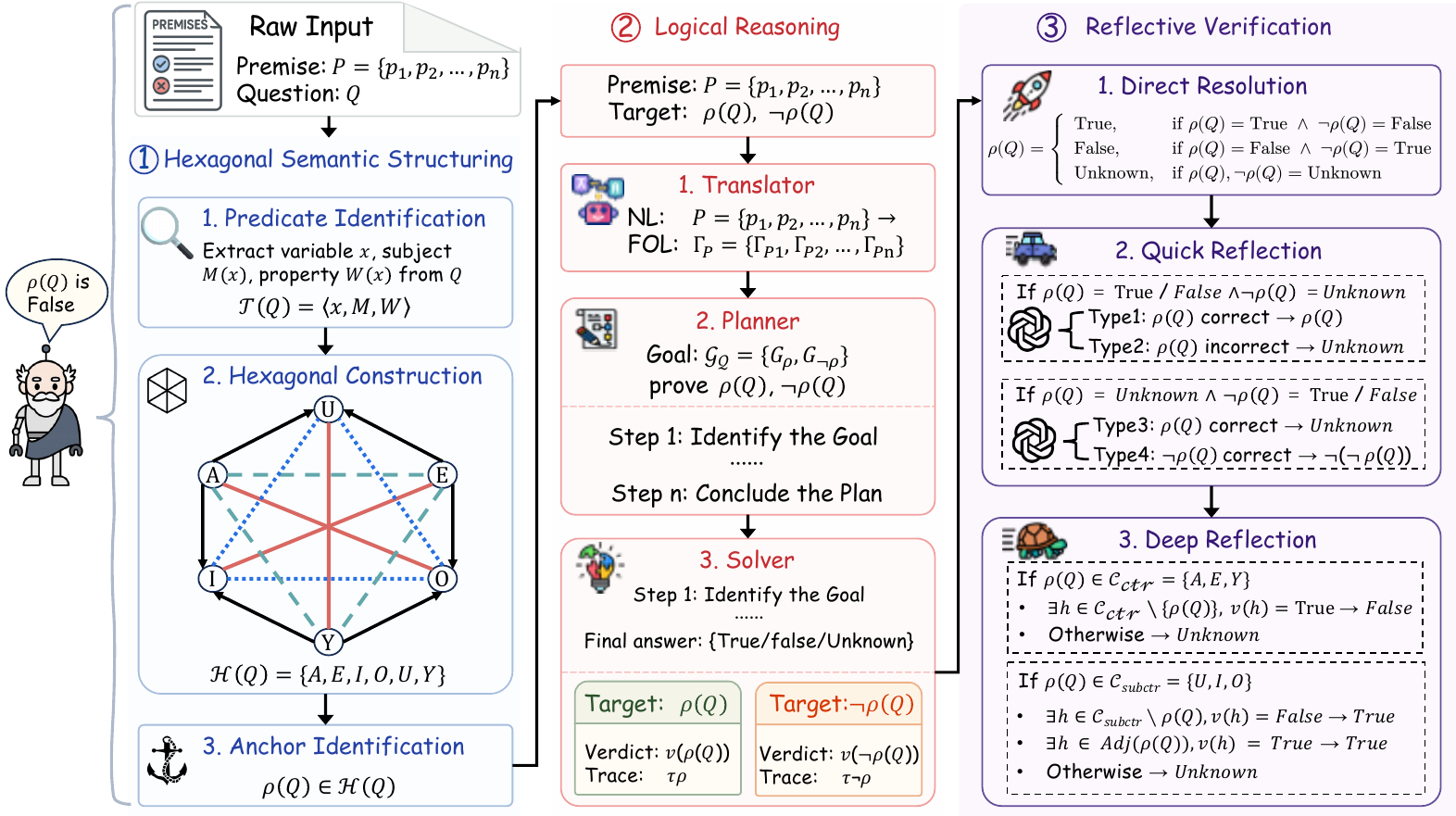}
    \caption{
    \textbf{Overview of HexLogicAgent.}
    The agent processes a natural-language reasoning problem through three stages.
    (1) \textbf{Hexagonal Semantic Structuring} extracts the predicate tuple
    \(\mathcal{T}(Q)=\langle x,M,W\rangle\), constructs the logical hexagon
    \(\mathcal{H}(Q)=\{A,E,I,O,U,Y\}\), and maps the target proposition to its anchor position \(\rho(Q)\).
    (2) \textbf{Logical Reasoning} formalizes the premises into \(\Gamma_P\), plans two proof goals for \(\rho(Q)\) and \(\lnot\rho(Q)\), and derives preliminary verdicts with reasoning traces.
    (3) \textbf{Reflective Verification} applies direct resolution, quick reflection, and deep reflection over the hexagonal relations to produce the final answer in
    \(\{\textbf{True},\textbf{False},\textbf{Unknown}\}\).
    }
    \label{fig:square-to-hexagon}
\end{figure}

\subsection{Task Definition}
\label{sec:task_definition}

Given a set of natural-language premises $P=\{p_1,p_2,\ldots,p_n\}$ and a target proposition $Q$, the task is to determine the truth status of $Q$ with respect to $P$.
The answer space is defined as
$\mathcal{A}=\{\text{True},\text{False},\text{Unknown}\}$.
Here, \textbf{True} indicates that $Q$ can be derived from $P$, \textbf{False} indicates that the contradiction of $Q$ can be derived from $P$, and \textbf{Unknown} indicates that neither $Q$ nor its contradiction can be conclusively derived.
\footnote{ProntoQA is restricted to the classical two-valued setting (\textbf{True}/\textbf{False}), whereas RepublicQA, ProofWriter, FOLIO, and ProverQA additionally include \textbf{Unknown} to handle indeterminate cases.}










\subsection{Hexagonal Semantic Structuring Stage}
\label{sec:hexagonal_semantic_structuring}

Natural-language propositions are often semantically compressed: a single surface statement may leave its quantificational scope, polarity, existential import, and counterexample space implicit.
If deduction starts directly from such an under-specified proposition, the reasoning process may prematurely commit to one semantic interpretation.
To mitigate this issue, Hex first maps the target proposition $Q$ into a logical hexagon before performing deduction.

\noindent\textbf{Predicate Identification.}
Given a target proposition $Q$, the system first identifies its normalized subject--property structure.
We define the basic elements as follows:
\begin{itemize}
\item \textbf{Individual variable $x$.}
The variable $x$ ranges over the relevant domain of discourse and denotes the object currently being evaluated.

\item \textbf{Subject predicate $M(x)$.}
The predicate $M(x)$ specifies the subject class under discussion. 
Depending on the proposition, it may denote entities, concepts, reified actions, or event instances, e.g., $Student(x)$, $Animal(x)$, $Action(x)$, or $Event(x)$.

\item \textbf{Property condition $W(x)$.}
The predicate $W(x)$ specifies the property or condition attributed to objects satisfying $M(x)$. 
It may correspond to states, attributes, normative meanings, evaluative meanings, or relation-based conditions, e.g., $Diligent(x)$, $Harmful(x)$, $Just(x)$, $Good(x)$, or $\exists y(Teacher(y)\land Help(x,y))$.
\end{itemize}

This step yields a normalized predicate tuple:
\begin{equation}
\mathcal{T}(Q)=\langle x,M,W\rangle .
\label{eq:predicate-tuple}
\end{equation}
For example, propositions such as `all $M$ are $W$'' and `some $M$ are $W$'' share the same predicate tuple: $M(x)$ identifies the subject class, while $W(x)$ identifies the asserted property or condition.
They differ only in quantificational force, corresponding to universal commitment and existential realization, respectively.
Thus, predicate identification abstracts away surface variation and provides the basis for constructing the hexagonal semantic space generated by $\langle x,M,W\rangle$.

\noindent\textbf{Hexagonal Position Construction.}
Following Proposition~\ref{prop:hexagonal-completion}, HexLogic instantiates the logical hexagon using the normalized tuple $\mathcal{T}(Q)=\langle x,M,W\rangle$.
The four basic positions are $A,E,I,O$, and the two composite positions are $U=A\lor E$ and $Y=I\land O$.
Together, they form the hexagonal semantic space of the target proposition:
\begin{equation}
\mathcal{H}(Q)=\{A,E,I,O,U,Y\}.
\label{eq:hexagonal-space}
\end{equation}
In this space, $A$ and $E$ capture universal alternatives, $I$ and $O$ capture existential alternatives, and $U,Y$ make uniform and mixed compositions explicit.

\noindent\textbf{Anchor Position Identification.}
Since the target proposition $Q$ may instantiate different logical forms, HexLogicAgent identifies its anchor position within the constructed hexagonal space.
The anchor mapping is defined as $\rho(Q)\in \mathcal{H}(Q)$,
where $\rho(Q)$ denotes the hexagonal position that best matches the normalized meaning of $Q$.
A universal affirmative proposition is mapped to $A$, a universal negative proposition to $E$, an existential affirmative proposition to $I$, and an existential negative proposition to $O$.
If $Q$ expresses a uniform alternative, it is mapped to $U$; if it expresses a mixed condition where both alternatives are realized, it is mapped to $Y$.

\noindent\textbf{Hexagonal Constraint Registration.}
After constructing $\mathcal{H}(Q)$ and identifying $\rho(Q)$, HexLogicAgent does not treat the six positions as independent propositions.
Instead, it registers the hexagonal relations established in Proposition~\ref{prop:hexagonal-completion} as structural constraints for downstream reasoning and verification.
Since the positions are instantiated from the same normalized tuple $\langle x,M,W\rangle$, their contradiction, contrariety, subcontrariety, and subalternation relations are guaranteed by construction under existential import.
Formally, the registered constraint set is defined as
\begin{equation}
\mathcal{R}_{\mathrm{hex}}
=
\{\mathcal{R}{\mathrm{con}},
\mathcal{R}{\mathrm{ctr}},
\mathcal{R}{\mathrm{subctr}},
\mathcal{R}{\mathrm{subalt}}\},
\label{eq}
\end{equation}
where $\mathcal{R}*{\mathrm{con}}$, $\mathcal{R}*{\mathrm{ctr}}$, $\mathcal{R}*{\mathrm{subctr}}$, and $\mathcal{R}*{\mathrm{subalt}}$ correspond to contradiction, contrariety, subcontrariety, and subalternation constraints, respectively.
These constraints are not used to re-prove the hexagon itself; rather, they are used later to check whether the solver's position-level judgments are mutually coherent.

\subsection{Logical Reasoning Stage}
\label{sec:logical_reasoning_stage}

This stage performs logical deduction over the hexagonal semantic space constructed in the previous stage.
It consists of three functional units: a \textit{Translator} for formalization, a \textit{Planner} for goal and path construction, and a \textit{Solver} for deductive reasoning.

\noindent\textbf{Translator.} The translator converts the natural-language premises $P=\{p_1,p_2,\ldots,p_n\}$ into a formal premise set $\Gamma_P$ in FOL. Unlike the predicate identification step, which normalizes the target proposition $Q$ into the subject--property tuple $\mathcal{T}(Q)=\langle x,M,W\rangle$, the translator focuses on formalizing the contextual premises that support or refute the hexagonal positions. The translated premises provide the logical context under which $A,E,I,O,U,Y$ are evaluated. Each translated premise is checked by a grammar-based validator to ensure syntactic well-formedness. The resulting formal premise set $\Gamma_P$ is then passed to the planner and solver for downstream deduction.

\noindent\textbf{Planner.}
The planner constructs reasoning blueprints for the anchor position $\rho(Q)$ and its negation $\lnot\rho(Q)$.
Here, $\rho(Q)$ denotes the hexagonal position that matches the normalized meaning of the target proposition, while $\lnot\rho(Q)$ serves as its contradictory target.
The planner therefore builds two primary reasoning goals:
\begin{equation}
\mathcal{G}_{Q}
=
\{G_{\rho},G_{\lnot\rho}\},
\label{eq:anchor-goals}
\end{equation}
where $G_{\rho}$ aims to prove $\rho(Q)$ and $G_{\lnot\rho}$ aims to prove $\lnot\rho(Q)$ under the formal premise set $\Gamma_P$.
For each goal, the planner selects relevant premises from $\Gamma_P$ and identifies applicable reasoning rules, such as Modus Ponens, Modus Tollens, conjunction, contradiction, and existential instantiation.
The remaining hexagonal relations are reserved for the reflective verification stage, where they are used for consistency checking and conflict resolution.

\noindent\textbf{Solver.}
The solver executes the two reasoning goals generated by the planner under the formal premise set $\Gamma_P$, namely the anchor proposition $\rho(Q)$ and its negation $\lnot\rho(Q)$.
For each target $g\in\{\rho(Q),\lnot\rho(Q)\}$, it applies logical reasoning rules and returns a verdict $v(g)\in{\text{True},\text{False},\text{Unknown}}$.
Here, \textbf{True} means that $g$ can be derived from $\Gamma_P$, \textbf{False} means that $g$ is refuted, and \textbf{Unknown} means that neither side is conclusive.
The solver outputs
\begin{equation}
V_Q=
\{(\rho(Q),v(\rho(Q)),\tau_{\rho}),
(\lnot\rho(Q),v(\lnot\rho(Q)),\tau_{\lnot\rho})\},
\label{eq}
\end{equation}
where $\tau_{\rho}$ and $\tau_{\lnot\rho}$ are the corresponding reasoning traces.
These results are then passed to the reflective verification stage for consistency checking and conflict resolution.

\subsection{Reflective Verification Stage}

This stage determines the final judgment through a three-stage reflective process that ensures coherence among the answers of the semiotic hexagon's six propositions.

\noindent\textbf{Direct Resolution.}
When $\rho(Q)$ and $\lnot\rho(Q)$ produce complementary verdicts, such as \( \rho(Q) = \text{True} \) and \( \lnot\rho(Q) = \text{False} \), the stage directly adopts the answer of $\rho(Q)$ as final. This scenario reflects a decisive and non-contradictory judgment grounded in the strict contradiction relationship between the proposition and its contradictory. The decision rule for this resolution strategy is defined as follows:
\begin{equation}\small
\label{equation2}
\rho(Q) =
\left\{
\begin{array}{ll}
\text{True}, & \text{if } \rho(Q) = \text{True} \; \land \; \lnot\rho(Q) = \text{False} \\
\text{False}, & \text{if } \rho(Q) = \text{False} \; \land \; \lnot\rho(Q) = \text{True} \\
\text{Unknown}, & \text{if } \rho(Q), \lnot\rho(Q) = \text{Unknown}
\end{array}
\right.
\end{equation}
If the two verdicts are not complementary, the system enters reflective verification and uses the hexagonal constraint set $\mathcal{R}{\mathrm{hex}}$ to identify possible inconsistency, missing evidence, or semantic conflict.

\noindent\textbf{Quick Reflection.}
When either $\rho(Q)$ or its contradictory $\lnot\rho(Q)$ is labeled as \textit{Unknown}, the stage triggers quick reflection by forwarding the two verdicts and their reasoning traces into a large language model.
This design leverages the observation that verification is often easier than generation: given candidate reasoning traces, an LLM can more reliably inspect consistency, identify missing steps, and detect unsupported conclusions than generate a correct proof from scratch~\cite{NPvsP-2023cook,gen-verifi-gap-2025ICLR}.
The model then analyzes the internal consistency of the deduction process and returns a refined judgment based on four reflection types:

\vspace{4pt}

{\small
\noindent
\setlength{\fboxsep}{6pt}
\colorbox{gray!10}{
\begin{minipage}[t]{0.47\linewidth}
\noindent
\texttt{Case 1:} If \( \rho(Q) = \text{True/False} \; \land \; \lnot\rho(Q) = \text{Unknown} \)

\begin{itemize}
  \item \textbf{Type 1:} \(\rho(Q)\) correct 
        $\;\Rightarrow\;$ Return \(\rho(Q)=\rho(Q)\)
  \item \textbf{Type 2:} \(\rho(Q)\) incorrect 
        $\;\Rightarrow\;$ Return \(\rho(Q)=\text{Unknown}\)
\end{itemize}
\end{minipage}\hfill
\begin{minipage}[t]{0.5\linewidth}
\noindent
\texttt{Case 2:} If \( \rho(Q) = \text{Unknown} \; \land \; \lnot\rho(Q) = \text{True/False} \)

\begin{itemize}
  \item \textbf{Type 3:} \(\rho(Q)\) correct 
        $\;\Rightarrow\;$ Return \(\rho(Q)=\text{Unknown}\)
  \item \textbf{Type 4:} \(\lnot\rho(Q)\) correct 
        $\;\Rightarrow\;$ Return \(\rho(Q)=\lnot(\lnot\rho(Q))\)
\end{itemize}
\end{minipage}
}
}

\vspace{4pt}

\noindent\textbf{Deep Reflection.}
When both $\rho(Q)$ and its contradictory $\lnot\rho(Q)$ receive the same decisive verdict,
e.g., both \textbf{True} or both \textbf{False}, the system enters
\textit{Deep Reflection} mode.
This mode expands verification from the two conflicting verdicts to the semantic neighborhood of the anchor position under the hexagonal constraint set $\mathcal{R}_{\mathrm{hex}}$.
We divide the six positions into the contrary group
$\mathcal{C}_{\mathrm{ctr}}=\{A,E,Y\}$, where no two positions should be jointly true,
and the subcontrary group
$\mathcal{C}_{\mathrm{subctr}}=\{U,I,O\}$, where no two positions should be jointly false.
For subalternation-based checking, the adjacent superordinate positions of $U,I,O$ are defined as
$
\mathrm{Adj}(U)=\{A,E\},
\mathrm{Adj}(I)=\{A,Y\},
\mathrm{Adj}(O)=\{E,Y\}.
$
The final reflective rule is defined as follows:

\vspace{4pt}

{\small
\noindent
\setlength{\fboxsep}{6pt}
\colorbox{gray!10}{
\begin{minipage}{0.95\linewidth}

\texttt{Case 1:} $\rho(Q)\in\mathcal{C}_{\mathrm{ctr}}$
\begin{itemize}
  \item $\exists h\in\mathcal{C}_{\mathrm{ctr}}\setminus\{\rho(Q)\},
  \ v(h)=\text{True}$
  $\;\Rightarrow\;$ Return $\rho(Q)=\text{False}$. \qquad
  {\footnotesize\textit{(\% by Rule~(1) in Prop.~\ref{prop:hexagonal-completion}: contrariety.)}}

  \item Otherwise
  $\;\Rightarrow\;$ Return $\rho(Q)=\text{Unknown}$.
\end{itemize}

\medskip

\texttt{Case 2:} $\rho(Q)\in\mathcal{C}_{\mathrm{subctr}}$
\begin{itemize}
  \item $\exists h\in\mathcal{C}_{\mathrm{subctr}}\setminus\{\rho(Q)\},
  \ v(h)=\text{False}$
  $\;\Rightarrow\;$ Return $\rho(Q)=\text{True}$. \quad
  {\footnotesize\textit{(\% by Rule~(2) in Prop.~\ref{prop:hexagonal-completion}: subcontrariety.)}}

  \item $\exists h\in\mathrm{Adj}(\rho(Q)),
  \ v(h)=\text{True}$
  $\;\Rightarrow\;$ Return $\rho(Q)=\text{True}$. \qquad \quad
  {\footnotesize\textit{(\% by Rule~(4) in Prop.~\ref{prop:hexagonal-completion}: subalternation.)}}

  \item Otherwise
  $\;\Rightarrow\;$ Return $\rho(Q)=\text{Unknown}$.
\end{itemize}

\end{minipage}
}
}

\vspace{4pt}

Overall, the reflective verification stage provides a layered mechanism for converting local verifier outputs into a coherent final judgment. Direct resolution handles clear contradiction-based cases, quick reflection re-examines asymmetric uncertainty through reasoning-trace inspection, and deep reflection resolves decisive conflicts by consulting the broader hexagonal neighborhood. In this way, the system does not treat each proposition in isolation, but enforces consistency among contradiction, contrary, subcontrary, and subalternation relations before returning \textbf{True}, \textbf{False}, or \textbf{Unknown}.

\definecolor{CMTgreen}{HTML}{008000}

\algtext*{EndFor}

\newcommand{\cmt}[1]{\textcolor{CMTgreen}{\quad $\triangleright$ #1}}

\begin{algorithm}[H]
\caption{\textbf{Hexagon-Guided Reflective Reasoning}}
\label{alg:hexlogicagent}
\small
\begin{algorithmic}[1]
\Require Natural-language premises $P=\{p_1,\ldots,p_n\}$ and target proposition $Q$
\Ensure Final answer $a\in\{\text{True},\text{False},\text{Unknown}\}$

\State \textcolor{CMTgreen}{// Stage 1: Hexagonal semantic structuring}
\State $(\mathcal{T}(Q),\mathcal{H}(Q),\rho(Q),\mathcal{R}_{\mathrm{hex}})
\leftarrow \textsc{HexStruct}(Q)$
\cmt{predicate tuple, hexagon, anchor, constraints}

\Statex
\State \textcolor{CMTgreen}{// Stage 2: Logical reasoning}
\State $\Gamma_P\leftarrow \textsc{Translator}(P)$
\cmt{formalize premises}
\State $\mathcal{G}_Q\leftarrow
\textsc{Planner}(\Gamma_P,\rho(Q),\lnot\rho(Q))$
\cmt{construct anchor and contradictory goals}
\ForAll{$g\in\mathcal{G}_Q$}
    \State $(v(g),\tau_g)\leftarrow \textsc{Solver}(\Gamma_P,g)$
    \cmt{obtain verdict and trace}
\EndFor
\State $V_Q\leftarrow
\{(\rho(Q),v(\rho(Q)),\tau_{\rho}),
(\lnot\rho(Q),v(\lnot\rho(Q)),\tau_{\lnot\rho})\}$

\Statex
\State \textcolor{CMTgreen}{// Stage 3: Hexagon-guided reflective verification}
\State $\omega\leftarrow \textsc{SelectMode}(V_Q)$
\cmt{$\omega\in\{\textsc{Direct},\textsc{Quick},\textsc{Deep}\}$}
\State $a\leftarrow
\textsc{ReflectVerify}(V_Q,\mathcal{H}(Q),\mathcal{R}_{\mathrm{hex}},\omega)$
\cmt{apply the selected verification strategy}

\State \Return $a$
\end{algorithmic}
\end{algorithm}

\section{Experiments}
\label{sec:experiments_analysis}

In this section, we conduct extensive experiments to evaluate the effectiveness, robustness, and internal design of HexLogicAgent. Specifically, we aim to answer the following research questions (RQs):

\begin{itemize}
    \item \textbf{RQ1:} How does HexLogicAgent perform against representative linear, aggregate, symbolic, and semiotic reasoning methods across different benchmarks and backbone models? (Section~\ref{sec:main-results})
    
    \item \textbf{RQ2:} How does each component of HexLogicAgent contribute to the final performance, including hexagonal semantic structuring and reflective verification? (Section~\ref{sec:ablation})
    
    \item \textbf{RQ3:} How does HexLogicAgent behave under increasing logical complexity, and what insights emerge from error and efficiency analyses? (Section~\ref{sec:analysis})
\end{itemize}

In the remainder of this section, we first introduce the experimental settings, including benchmark datasets, backbone models, and baseline methods(Section~\ref{sec:experimental-settings}). Finally, we provide qualitative case studies to illustrate how hexagonal semantic structuring and reflective verification affect the reasoning process (Appendix).

\subsection{Experimental Settings}
\label{sec:experimental-settings}

\paragraph{Benchmarks.}
We evaluate HexLogicAgent on five logical reasoning benchmarks:
FOLIO~\cite{Benchmark-FOLIO2022folio}, ProntoQA~\cite{Benchmark-Pronto-2022language}, ProofWriter~\cite{Benchmark-ProofWriter-2020proofwriter}, ProverQA~\cite{Benchmark-ProverQA-ICLR-2025large}, and RepublicQA~\cite{zhang2026semanticawarelogicalreasoningsemiotic}.
These benchmarks cover three complementary scenarios:

\begin{enumerate}
    \item \textbf{Controlled deductive reasoning.}
    ProntoQA evaluates multi-hop deduction over synthetic proof chains, while ProofWriter tests rule-based reasoning under open-world assumptions.

    \item \textbf{Natural-language formal reasoning.}
    FOLIO provides expert-curated natural-language problems annotated with first-order logic, and ProverQA further combines natural-language statements, FOL translations, and theorem-prover-verified reasoning chains.

    \item \textbf{Abstract philosophical reasoning.}
    RepublicQA, introduced in the LogicAgent setting, evaluates reasoning over abstract philosophical propositions derived from classical argumentative contexts, testing semantic ambiguity, normative concepts, and counterargument structures.
\end{enumerate}

\paragraph{Backbone models.}
We conduct experiments on three backbone models to examine whether HexLogicAgent generalizes across different model families and deployment settings:

\begin{enumerate}
    \item \textbf{Closed-source model:} DeepSeek-V3.2, used as a representative closed-source backbone.

    \item \textbf{Open-source model:} Qwen3-30B-A3B, used as a representative open-source backbone.

    \item \textbf{LogicAgent-aligned backbone:} Qwen2.5-32B, used to maintain consistency with the original LogicAgent evaluation, while also forming an intra-family comparison with Qwen3-30B-A3B.
\end{enumerate}
This setup allows us to test both general cross-model effectiveness and direct improvement over the previous semiotic-square-based framework under the same backbone.

\definecolor{gainred}{HTML}{B22222}
\definecolor{gainbg}{HTML}{FDEDEC}
\definecolor{secondgreen}{HTML}{008000}

\newcommand{\second}[1]{\textcolor{secondgreen}{\underline{#1}}}

\newcommand{\yes}{\textcolor{secondgreen}{$\checkmark$}}
\newcommand{\no}{\textcolor{gainred}{$\times$}}

\newcommand{\bestmethod}[1]{\textbf{#1}}
\newcommand{\bestcell}[1]{\textbf{#1}}

\newcommand{\deltamethod}{\textcolor{gainred}{$\triangle$}}
\newcommand{\deltagain}[1]{\textcolor{gainred}{+#1}}

\newcommand{\modelicon}[1]{%
  \raisebox{-0.12em}{\includegraphics[height=1.15em]{#1}}%
}
\newcommand{\deepseek}{%
  \raisebox{-0.12em}{%
    \includegraphics[height=1.15em]{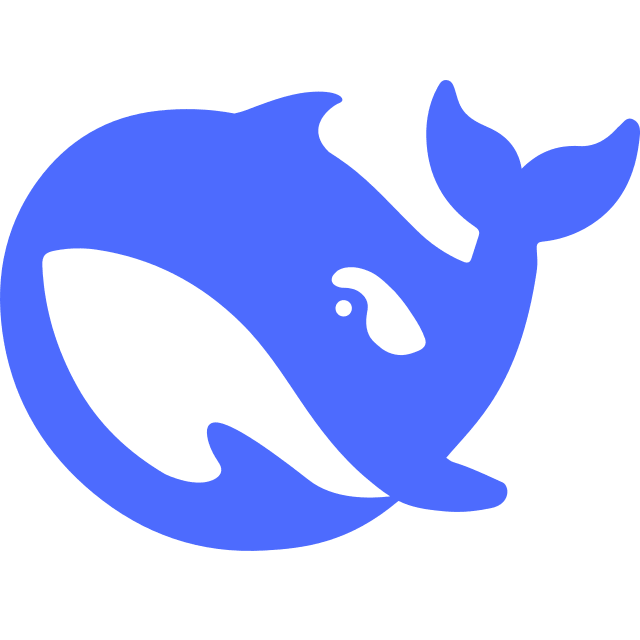}%
  }%
}

\newcommand{\qwen}{%
  \raisebox{-0.12em}{%
    \includegraphics[height=1.15em]{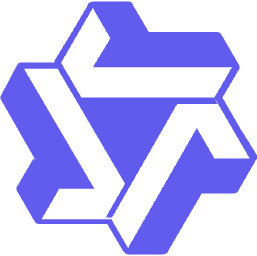}%
  }%
}

\begin{table*}[t]
\centering
\footnotesize
\caption{
\textbf{Main results across three backbone models.}
Each method is characterized by four reasoning features: linear reasoning (LR),
aggregative reasoning (AR), symbolic reasoning (SR), and semiotic reasoning (SeR).
AVG. is the macro-average accuracy over the five benchmarks.
\textbf{Bold} values mark the best results;
\textcolor{secondgreen}{\underline{green underlines}} mark the second-best results;
\colorbox{gainbg}{red-shaded rows} mark HexLogicAgent and
\textcolor{gainred}{$\triangle$} shows its absolute gain over the second-best result.
}

\begin{tabular*}{\linewidth}{@{\extracolsep{\fill}}lcccc l ccccc c@{}}
\toprule
\textbf{Model}
& \multicolumn{4}{c}{\textbf{Feature}}
& \textbf{Method}
& \multicolumn{5}{c}{\textbf{Benchmarks} $\uparrow$}
& \textbf{AVG.} $\uparrow$ \\
\cmidrule(lr){2-5}
\cmidrule(lr){7-11}
& \textbf{LR}
& \textbf{AR}
& \textbf{SR}
& \textbf{SeR}
&
& \textbf{FOLIO}
& \textbf{ProntoQA}
& \makecell{\textbf{Proof}\\\textbf{Writer}}
& \textbf{ProverQA}
& \textbf{RepublicQA}
& \\
\midrule

\multirow{11}{*}{%
  \makecell[c]{%
    \deepseek\\
    DeepSeek\\
    V3.2%
  }%
}
& \yes & \no & \no & \no
& Direct
& 64.22 & 60.60 & 47.83 & 58.00 & 74.00 & 60.93 \\
& \yes & \no & \no & \no
& CoT
& 79.90 & 100.00 & 87.17 & 60.20 & 66.50 & 78.75 \\

& \yes & \yes & \no & \no
& ToT
& 83.82 & 98.40 & 91.67 & 64.60 & 64.00 & 80.50 \\
& \yes & \yes & \no & \no
& CR
& 83.35 & 99.00 & 90.50 & 63.80 & 64.00 & 80.13 \\
& \yes & \yes & \no & \no
& DetermLR
& 82.84 & 100.00 & 92.00 & 66.60 & 66.00 & 81.49 \\

& \yes & \no & \yes & \no
& LogicLM
& 69.12 & 98.80 & 77.00 & 74.20 & 37.50 & 71.32 \\
& \yes & \no & \yes & \no
& SymbCoT
& 85.29 & 97.71 & 91.33 & 76.00 & 36.00 & 77.27 \\
& \yes & \yes & \yes & \no
& Aristotle
& 79.41 & 100.00 & 90.50 & 74.50 & 77.50 & 84.38 \\
\cmidrule(lr){6-12}

& \yes & \yes & \yes & \yes
& LogicAgent
& \second{85.78} & \second{100.00} & \second{92.16} & \second{76.40} & \second{79.50} & \second{86.77} \\

\rowcolor{gainbg}
\cellcolor{white}
& \cellcolor{white}\yes
& \cellcolor{white}\yes
& \cellcolor{white}\yes
& \cellcolor{white}\yes
& \bestmethod{HexLogicAgent}
& \bestcell{87.74}
& \bestcell{100.00}
& \bestcell{95.50}
& \bestcell{82.40}
& \bestcell{81.50}
& \bestcell{89.43} \\

\rowcolor{gainbg}
\cellcolor{white}
& \cellcolor{white}
& \cellcolor{white}
& \cellcolor{white}
& \cellcolor{white}
& \deltamethod
& \deltagain{1.96}
& \deltagain{0.00}
& \deltagain{3.34}
& \deltagain{6.00}
& \deltagain{2.00}
& \deltagain{2.66} \\

\midrule

\multirow{11}{*}{%
  \makecell[c]{%
    \qwen\\
    Qwen2.5\\32B
  }%
}
& \yes & \no & \no & \no
& Direct
& 60.29 & 82.00 & 59.17 & 39.60 & 68.50 & 61.91 \\
& \yes & \no & \no & \no
& CoT
& 68.42 & 92.40 & 63.17 & 47.20 & 72.00 & 68.64 \\

& \yes & \yes & \no & \no
& ToT
& 72.54 & 82.50 & 64.40 & 53.40 & 56.00 & 65.77 \\
& \yes & \yes & \no & \no
& CR
& 71.57 & 80.20 & 58.33 & 51.80 & 57.00 & 63.78 \\
& \yes & \yes & \no & \no
& DetermLR
& 71.56 & 96.50 & 65.33 & 63.90 & 71.50 & 73.76 \\

& \yes & \no & \yes & \no
& LogicLM
& 71.93 & 91.89 & 63.82 & 62.40 & 70.00 & 72.01 \\
& \yes & \no & \yes & \no
& SymbCoT
& 70.59 & 95.20 & 64.67 & 57.20 & 76.00 & 72.73 \\
& \yes & \yes & \yes & \no
& Aristotle
& 68.68 & 94.80 & 63.23 & 56.20 & 74.50 & 71.48 \\
\cmidrule(lr){6-12}

& \yes & \yes & \yes & \yes
& LogicAgent
& \second{79.90} & \second{97.80} & \second{71.95} & \second{68.60} & \second{82.50} & \second{80.15} \\

\rowcolor{gainbg}
\cellcolor{white}
& \cellcolor{white}\yes
& \cellcolor{white}\yes
& \cellcolor{white}\yes
& \cellcolor{white}\yes
& \bestmethod{HexLogicAgent}
& \bestcell{81.37}
& \bestcell{99.20}
& \bestcell{75.00}
& \bestcell{72.40}
& \bestcell{86.50}
& \bestcell{82.89} \\

\rowcolor{gainbg}
\cellcolor{white}
& \cellcolor{white}
& \cellcolor{white}
& \cellcolor{white}
& \cellcolor{white}
& \deltamethod
& \deltagain{1.47}
& \deltagain{1.40}
& \deltagain{3.05}
& \deltagain{3.80}
& \deltagain{4.00}
& \deltagain{2.74} \\

\midrule

\multirow{11}{*}{%
  \makecell[c]{%
    \qwen\\
    Qwen3\\
    30B-A3B%
  }%
}
& \yes & \no & \no & \no
& Direct
& 59.80 & 82.00 & 54.83 & 49.00 & 71.00 & 63.33 \\
& \yes & \no & \no & \no
& CoT
& 76.96 & 100.00 & 78.50 & 56.60 & 65.50 & 75.51 \\

& \yes & \yes & \no & \no
& ToT
& 64.71 & 100.00 & 82.00 & 70.50 & 75.00 & 78.44 \\
& \yes & \yes & \no & \no
& CR
& 67.15 & 100.00 & 84.00 & 71.80 & 74.50 & 79.49 \\
& \yes & \yes & \no & \no
& DetermLR
& 79.90 & 100.00 & \second{84.17} & 68.20 & 72.50 & 80.95 \\

& \yes & \no & \yes & \no
& LogicLM
& 62.25 & 48.60 & 32.50 & 56.80 & 39.00 & 47.83 \\
& \yes & \no & \yes & \no
& SymbCoT
& \second{82.27} & 100.00 & 81.00 & 73.75 & 35.00 & 74.40 \\
& \yes & \yes & \yes & \no
& Aristotle
& 77.45 & 100.00 & 79.50 & 72.80 & 74.50 & 80.85 \\
\cmidrule(lr){6-12}

& \yes & \yes & \yes & \yes
& LogicAgent
& 81.86 & \second{100.00} & \second{84.17} & \second{75.20} & \second{76.50} & \second{83.55} \\

\rowcolor{gainbg}
\cellcolor{white}
& \cellcolor{white}\yes
& \cellcolor{white}\yes
& \cellcolor{white}\yes
& \cellcolor{white}\yes
& \bestmethod{HexLogicAgent}
& \bestcell{85.29}
& \bestcell{100.00}
& \bestcell{89.00}
& \bestcell{77.00}
& \bestcell{78.50}
& \bestcell{85.96} \\

\rowcolor{gainbg}
\cellcolor{white}
& \cellcolor{white}
& \cellcolor{white}
& \cellcolor{white}
& \cellcolor{white}
& \deltamethod
& \deltagain{3.02}
& \deltagain{0.00}
& \deltagain{4.83}
& \deltagain{1.80}
& \deltagain{2.00}
& \deltagain{2.41} \\

\bottomrule
\end{tabular*}
\label{tab:main_results_all_models}
\end{table*}

\paragraph{Baselines.}
We compare HexLogicAgent with nine representative baselines, characterized by four reasoning features: linear reasoning (LR), aggregative reasoning (AR), symbolic reasoning (SR), and semiotic reasoning (SeR).

\begin{enumerate}
    \item \textbf{Linear reasoning:}
    Direct prompting and Chain-of-Thought (CoT)~\cite{COT-2022chain}, which follow a single reasoning path.

    \item \textbf{Aggregative Reasoning:}
    Tree-of-Thought (ToT)~\cite{TOT-2023tree}, Cumulative Reasoning (CR)~\cite{Cumulative-reasoning-2023-TsinghuIIIS}, and DetermLR~\cite{Determlr-ACL-2024-RenDaGaoling}, which explore or refine multiple reasoning paths.

    \item \textbf{Symbolic reasoning:}
    Logic-LM~\cite{LogicLM-2023}, SymbCoT~\cite{Symbolic-COT-2024faithful}, and Aristotle~\cite{Baseline-Aristotle2024xu}, which introduce symbolic representations, logical decomposition, or formal constraints.

    \item \textbf{Semiotic reasoning:}
    LogicAgent~\cite{zhang2026semanticawarelogicalreasoningsemiotic}, the most directly related baseline, which organizes proposition meanings through the semiotic square.
\end{enumerate}


\begin{figure}[width=\textwidth,pos=h!]
    \centering
    \includegraphics[
        width=0.65\textwidth,
        height=0.4\textheight,
        keepaspectratio
    ]{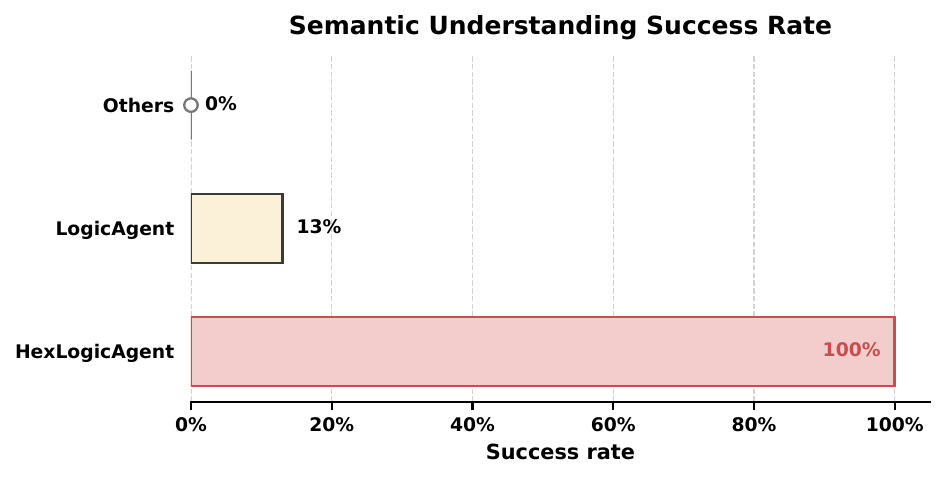}
    \caption{
    \textbf{Semantic-understanding success rate.}
    The figure compares the semantic-understanding success rates of the other reasoning methods, LogicAgent, and HexLogicAgent, which achieve $0\%$, $13\%$, and $100\%$, respectively.
    }
    \label{fig:semantic_understanding}
\end{figure}

\subsection{Main Results}
\label{sec:main-results}

To answer RQ1, Table~\ref{tab:main_results_all_models} reports the main results across five benchmarks and three backbone models.

\textbf{(1) HexLogicAgent consistently achieves the best overall performance across all backbone models.}
This shows that the proposed hexagon-guided reasoning framework is effective beyond a single model family or deployment setting.
Specifically, HexLogicAgent obtains the highest macro-average accuracy on DeepSeek-V3.2, Qwen2.5-32B, and Qwen3-30B-A3B.
Compared with the strongest baseline in each setting, it improves the average accuracy by $+2.66$, $+2.74$, and $+2.41$ points, respectively.
These results indicate that the logical hexagon provides a stable semantic structure for improving LLM reasoning across different backbones.

\textbf{(2) HexLogicAgent consistently improves over LogicAgent by enhancing semantic understanding.}
This comparison is particularly important because LogicAgent already organizes proposition meanings using the semiotic square. Across all frontier LLMs, extending the semantic structure to the logical hexagon consistently improves reasoning accuracy, increasing the average accuracy from $86.77$ to $89.43$ on DeepSeek-V3.2, from $80.15$ to $82.89$ on Qwen2.5-32B, and from $83.55$ to $85.96$ on Qwen3-30B-A3B. Figure~\ref{fig:semantic_understanding} further explains this improvement. While conventional reasoning methods fail to correctly capture the semantic relations in these cases ($0\%$ success rate), and LogicAgent succeeds on only $13\%$ of them, HexLogicAgent achieves a $100\%$ success rate. These results indicate that introducing existential and mixed-state positions enables substantially more accurate semantic interpretation, which translates into more reliable downstream logical reasoning.

\textbf{(3) HexLogicAgent shows strong advantages on both structured reasoning and semantically abstract tasks.}
On benchmarks requiring formal or multi-step logical verification, the gains are particularly clear.
For example, on DeepSeek-V3.2, HexLogicAgent improves over LogicAgent by $+6.00$ points on ProverQA and $+3.34$ points on ProofWriter.
On Qwen3-30B-A3B, it improves by $+4.83$ points on ProofWriter and $+3.02$ points on FOLIO.
At the same time, HexLogicAgent remains competitive on semantically abstract reasoning tasks such as RepublicQA, which involve philosophical propositions and counterargument structures.
Across all three backbones, it achieves the best performance on RepublicQA, improving over the second-best result by $+2.00$, $+4.00$, and $+2.00$ points, respectively.
These results indicate that the logical hexagon enables both more complete verification in structured reasoning and better handling of abstract proposition meanings.

\textbf{Overall, the results answer RQ1 affirmatively.}
HexLogicAgent consistently outperforms all representative baselines across benchmarks and backbone models. Moreover, its consistent gains over LogicAgent demonstrate the advantage of the logical hexagon over the semiotic square, validating the effectiveness of hexagon-guided semantic structuring and reflective verification.

\begin{figure}[width=\textwidth,pos=h!]
    \centering
    \includegraphics[
        width=0.95\textwidth,
        height=0.5\textheight,
        keepaspectratio
    ]{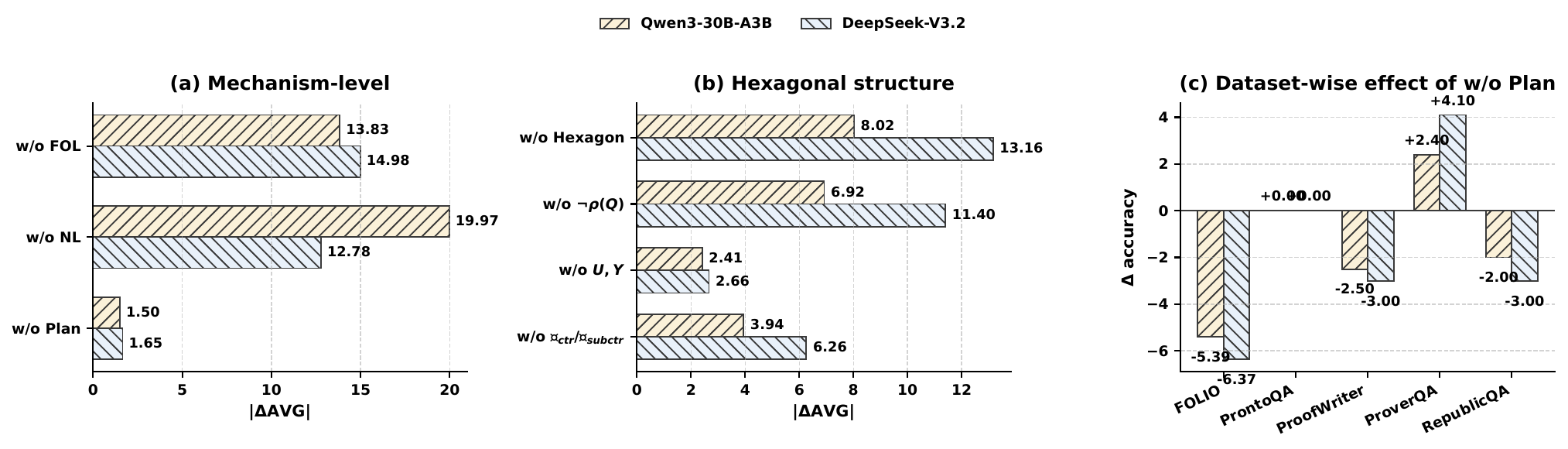}
    \caption{
    \textbf{Ablation drops and dataset-wise effect of planning removal.}
    Panels (a) and (b) show absolute average accuracy drops under mechanism-level and hexagonal-structure ablations, where larger $|\Delta\mathrm{AVG}|$ indicates a more important component.
    Panel (c) shows the dataset-wise accuracy change after removing the planning module, highlighting that w/o Plan improves ProverQA but hurts most other benchmarks.
    }
    \label{fig:ablation_three_panel}
\end{figure}

\definecolor{fullgreen}{HTML}{008000}
\definecolor{droppred}{HTML}{B22222}

\newcommand{\fullres}[1]{\textbf{\textcolor{fullgreen}{#1}}}
\newcommand{\dropres}[1]{\textcolor{droppred}{#1}}
\newcommand{\critdrop}[1]{\textbf{\textcolor{droppred}{#1}}}

\begin{table*}[t]
\centering
\small
\setlength{\tabcolsep}{4.8pt}
\renewcommand{\arraystretch}{1.10}
\begin{tabular}{llccccccc}
\toprule
\textbf{Model}
& \textbf{Setting}
& \textbf{FOLIO}
& \textbf{ProntoQA}
& \makecell{\textbf{Proof}\\\textbf{Writer}}
& \textbf{ProverQA}
& \textbf{RepublicQA}
& \textbf{AVG}
& \textbf{$\Delta$AVG} \\
\midrule

\multirow{4}{*}{%
  \makecell[c]{%
    \qwen\\
    Qwen3\\
    30B-A3B%
  }%
}
& \textbf{HexLogicAgent}
& \fullres{85.29}
& \fullres{100.00}
& \fullres{89.00}
& 77.00
& \fullres{78.50}
& \fullres{85.96}
& --- \\
& w/o FOL
& 72.55 & 84.20 & 82.67 & 49.20 & 72.00
& 72.12 & \dropres{-13.83} \\
& w/o NL
& 68.63 & 75.60 & 82.33 & 63.40 & 40.00
& 65.99 & \critdrop{-19.97} \\
& w/o Plan
& 79.90 & 100.00 & 86.50 & \fullres{79.40} & 76.50
& 84.46 & \dropres{-1.50} \\

\midrule

\multirow{4}{*}{%
  \makecell[c]{%
    \deepseek\\
    DeepSeek\\
    V3.2%
  }%
}
& \textbf{HexLogicAgent}
& \fullres{87.74}
& \fullres{100.00}
& \fullres{95.50}
& 82.40
& \fullres{81.50}
& \fullres{89.43}
& --- \\
& w/o FOL
& 69.12 & 90.40 & 85.33 & 58.40 & 69.00
& 74.45 & \critdrop{-14.98} \\
& w/o NL
& 73.53 & 96.20 & 93.33 & 80.20 & 40.00
& 76.65 & \dropres{-12.78} \\
& w/o Plan
& 81.37 & 100.00 & 92.50 & \fullres{86.50} & 78.50
& 87.77 & \dropres{-1.65} \\

\bottomrule
\end{tabular}
\caption{
\textbf{Mechanism-level ablation results.}
FOL, NL, and Plan denote formal-logic reasoning, natural-language semantic reasoning, and planning, respectively.
AVG denotes the macro-average accuracy over the five benchmarks, and $\Delta$AVG reports the absolute change from the full HexLogicAgent.
\textbf{\textcolor{droppred}{Bold red}} marks the largest average drop within each model block.
}
\label{tab:mechanism_ablation}
\end{table*}

\begin{table*}[t]
\centering
\small
\setlength{\tabcolsep}{4.8pt}
\renewcommand{\arraystretch}{1.10}
\begin{tabular}{llccccccc}
\toprule
\textbf{Model}
& \textbf{Setting}
& \textbf{FOLIO}
& \textbf{ProntoQA}
& \makecell{\textbf{Proof}\\\textbf{Writer}}
& \textbf{ProverQA}
& \textbf{RepublicQA}
& \textbf{AVG}
& \textbf{$\Delta$AVG} \\
\midrule

\multirow{5}{*}{%
  \makecell[c]{%
    \qwen\\
    Qwen3\\
    30B-A3B%
  }%
}
& \textbf{HexLogicAgent}
& \fullres{85.29}
& \fullres{100.00}
& \fullres{89.00}
& \fullres{77.00}
& \fullres{78.50}
& \fullres{85.96}
& --- \\
& w/o Hexagon
& 66.67 & 99.80 & 84.33 & 68.40 & 70.50
& 77.94 & \critdrop{-8.02} \\
& w/o $\neg\rho(Q)$
& 72.22 & 100.00 & 83.11 & 69.87 & 70.00
& 79.04 & \dropres{-6.92} \\
& w/o $U, Y$
& 81.86 & 100.00 & 84.17 & 75.20 & 76.50
& 83.55 & \dropres{-2.41} \\
& w/o $\mathcal{C}_{\mathrm{ctr}}/\mathcal{C}_{\mathrm{subctr}}$
& 78.92 & 99.80 & 83.17 & 75.20 & 73.00
& 82.02 & \dropres{-3.94}\\

\midrule

\multirow{5}{*}{%
  \makecell[c]{%
    \deepseek\\
    DeepSeek\\
    V3.2%
  }%
}
& \textbf{HexLogicAgent}
& \fullres{87.74}
& \fullres{100.00}
& \fullres{95.50}
& \fullres{82.40}
& \fullres{81.50}
& \fullres{89.43}
& --- \\
& w/o Hexagon
& 65.19 & 98.40 & 90.17 & 63.07 & 64.50
& 76.27 & \critdrop{-13.16} 
 \\
& w/o $\neg\rho(Q)$
& 62.25 & 98.60 & 90.67 & 67.60 & 71.00
& 78.02 & \dropres{-11.40} \\
& w/o $U, Y$
& 85.78 & 100.00 & 92.16 & 76.40 & 79.50
& 86.77 & \dropres{-2.66} \\
& w/o $\mathcal{C}_{\mathrm{ctr}}/\mathcal{C}_{\mathrm{subctr}}$
& 70.10 & 99.80 & 92.83 & 81.60 & 71.50
& 83.17 & \dropres{-6.26}\\

\bottomrule
\end{tabular}
\caption{
\textbf{Hexagonal-structure ablation results.}
$\neg\rho(Q)$ denotes the contradictory anchor target; $U/Y$ denotes the two composite hexagonal positions;
$\mathcal{C}_{\mathrm{ctr}}=\{A,E,Y\}$ and $\mathcal{C}_{\mathrm{subctr}}=\{U,I,O\}$ denote the contrary and subcontrary groups.
AVG denotes the macro-average accuracy over the five benchmarks, and $\Delta$AVG reports the absolute change from the full HexLogicAgent.
\textbf{\textcolor{droppred}{Bold red}} marks the largest average drop within each model block.
}
\label{tab:hexagonal_structure_ablation}
\end{table*}

\subsection{Ablation Study}
\label{sec:ablation}

To answer RQ2, we conduct ablation studies to examine the contribution of each component in HexLogicAgent, as shown in Tables~\ref{tab:mechanism_ablation}--\ref{tab:hexagonal_structure_ablation} and Figure~\ref{fig:ablation_three_panel}.

\textbf{(1) FOL and NL provide complementary grounding.}
Removing either FOL or NL leads to large drops on both backbones, showing that HexLogicAgent depends on both symbolic constraints and natural-language semantic associations.
On Qwen3-30B-A3B, removing NL causes the largest degradation, with AVG dropping from $85.96$ to $65.99$ ($-19.97$), while removing FOL reduces AVG to $72.12$ ($-13.83$).
On DeepSeek-V3.2, removing FOL has the strongest effect, decreasing AVG from $89.43$ to $74.45$ ($-14.98$), while removing NL drops AVG to $76.65$ ($-12.78$).
These results indicate that FOL improves inferential precision through formal constraints, whereas NL preserves contextual and semantic cues that are not fully captured by symbolic forms alone.

\textbf{(2) The logical hexagon is not merely an additional structure, but a useful relational constraint system.}
Removing the full hexagonal structure causes substantial performance degradation on both backbones: $85.96 \rightarrow 77.94$ ($-8.02$) on Qwen3-30B-A3B and $89.43 \rightarrow 76.27$ ($-13.16$) on DeepSeek-V3.2.
Among the fine-grained components, removing the contradictory anchor target $\neg\rho(Q)$ also causes large drops, from $85.96$ to $79.04$ ($-6.92$) and from $89.43$ to $78.02$ ($-11.40$), respectively.
The contrary/subcontrary group constraints further contribute to performance, especially on DeepSeek-V3.2, where removing $\mathcal{C}_{\mathrm{ctr}}/\mathcal{C}_{\mathrm{subctr}}$ reduces AVG from $89.43$ to $83.17$ ($-6.26$).
By contrast, removing the composite positions $U,Y$ leads to smaller drops ($-2.41$ and $-2.66$).
This suggests that the main benefit of the logical hexagon comes from enforcing structured semantic relations, especially contradiction, contrariety, and subcontrariety.

\textbf{(3) Planning is useful overall, but can be harmful for already long reasoning chains.}
Removing Plan slightly improves ProverQA on both backbones: $77.00 \rightarrow 79.40$ ($+2.40$) on Qwen3-30B-A3B and $82.40 \rightarrow 86.50$ ($+4.10$) on DeepSeek-V3.2.
However, this gain is not consistent across datasets.
Excluding ProverQA, removing Plan reduces the average accuracy from $88.20$ to $85.73$ ($-2.47$) on Qwen3-30B-A3B and from $91.19$ to $88.09$ ($-3.09$) on DeepSeek-V3.2.
This indicates that planning generally helps organize heterogeneous reasoning tasks, but when the original reasoning chain is already long, an extra planning stage may further extend the trajectory, increasing context burden, memory pressure, and error accumulation.
Therefore, planning should be adaptive rather than fixed for long-chain logical reasoning.

\textbf{Overall, the ablation results answer RQ2 affirmatively.}
Both the mechanism-level components and the hexagonal semantic structure contribute to the final performance.
FOL and NL offer complementary grounding signals, while the logical hexagon improves reasoning by enforcing structured relations among contradiction, contrariety, subcontrariety, and composite positions.

\begin{figure}[width=\textwidth,pos=h!]
    \centering
    \includegraphics[
        width=0.95\textwidth,
        height=0.4\textheight,
        keepaspectratio
    ]{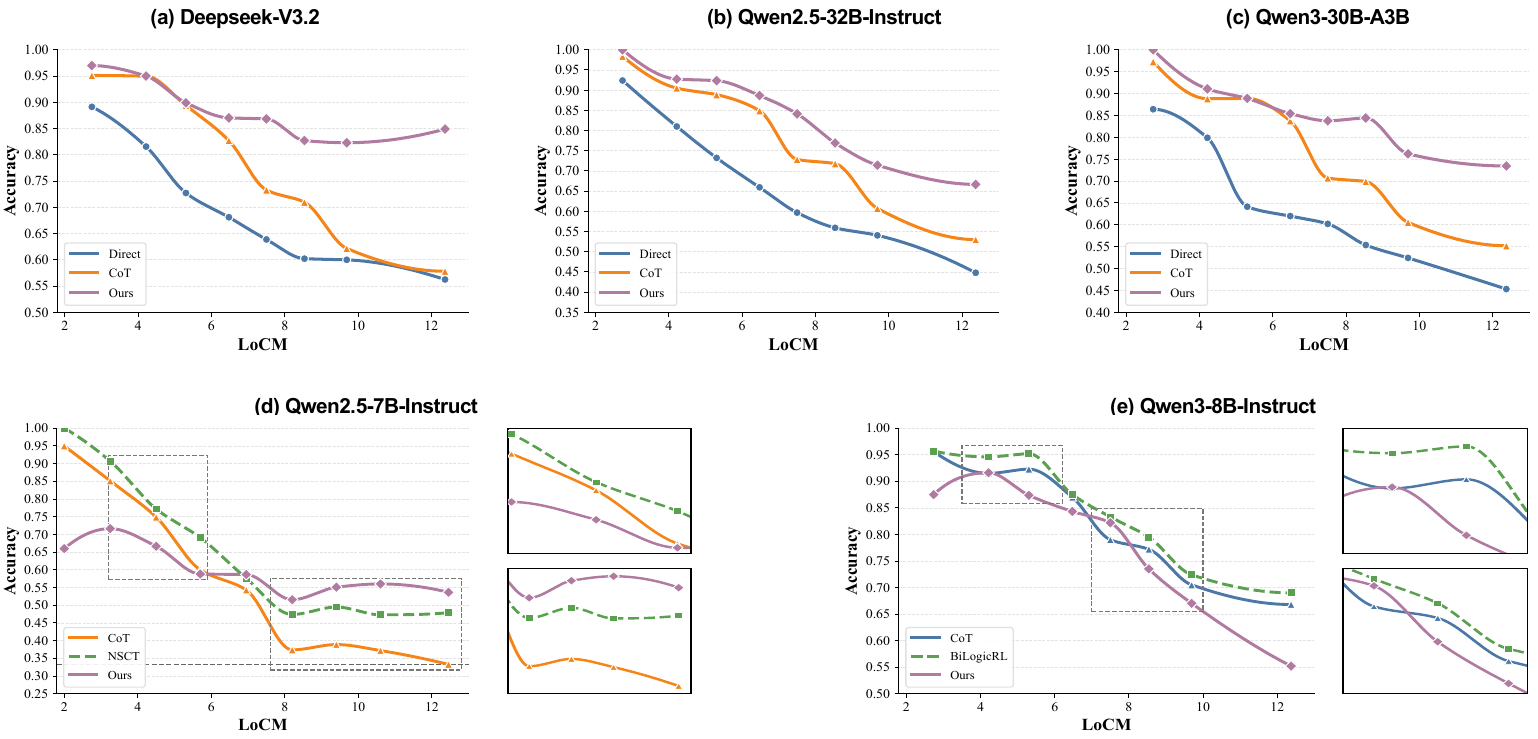}
    \caption{
    \textbf{Phase transition curves across LoCM.}
    Panels (a)–(c) compare Direct prompting, CoT, and LogicAgent on DeepSeek-V3.2, Qwen2.5-32B-Instruct, and Qwen3-30B-A3B, respectively. Panels (d)–(e) compare LogicAgent with representative training-based methods on Qwen2.5-7B-Instruct and Qwen3-8B-Instruct. Each point represents the accuracy within a LoCM bin, and the insets enlarge representative local regions.
    }
    \label{fig:phase_transition}
\end{figure}

\subsection{Analysis}
\label{sec:analysis}

To answer RQ3, we conduct three complementary analyses. \textit{(1) Phase-transition analysis} examines model performance under increasing logical complexity (LOCM). \textit{(2) Error analysis} characterizes common reasoning failures and investigates how the logical hexagon mitigates them. \textit{(3) Efficiency analysis} evaluates the computational overhead of HexLogicAgent.

\subsubsection{Phase-Transition Analysis}
Inspired by the notion of logical phase transitions~\cite{zhang2026logicalphasetransitionsunderstanding}, we analyze how model performance changes with increasing LoCM, as shown in Figure~\ref{fig:phase_transition}.

\textbf{(1) HexLogicAgent delays logical phase transitions on frontier LLMs.}
As logical complexity increases, all methods eventually experience performance degradation. However, Direct prompting and CoT undergo abrupt accuracy drops once LoCM exceeds the medium-complexity region, whereas HexLogicAgent maintains substantially higher accuracy in the high-complexity regime. Notably, neither advanced prompting strategies nor representative training-based methods fundamentally delay the transition point. In contrast, HexLogicAgent consistently postpones the onset of logical phase transitions, suggesting that \textit{explicitly guiding reasoning with logical hexagon structures is an effective mechanism for improving robustness under increasing logical complexity}.

\textbf{(2) The benefits of HexLogicAgent become more pronounced as logical complexity increases.}
The performance gap between HexLogicAgent and conventional prompting methods is relatively modest on low-complexity problems but widens progressively in high-complexity regions. This suggests that explicitly modeling richer logical relations is particularly beneficial for solving structurally complex reasoning tasks, where conventional prompting begins to fail.

\textbf{(3) The gains diminish on smaller models.}
On Qwen2.5-7B and Qwen3-8B, HexLogicAgent no longer consistently outperforms training-based baselines. This is likely because the framework relies on strong instruction following, reliable FOL translation, multi-stage reasoning, and long-context processing, all of which remain challenging for smaller models.

\begin{figure}[width=\textwidth,pos=h!]
    \centering
    \includegraphics[
        width=0.95\textwidth,
        height=0.4\textheight,
        keepaspectratio
    ]{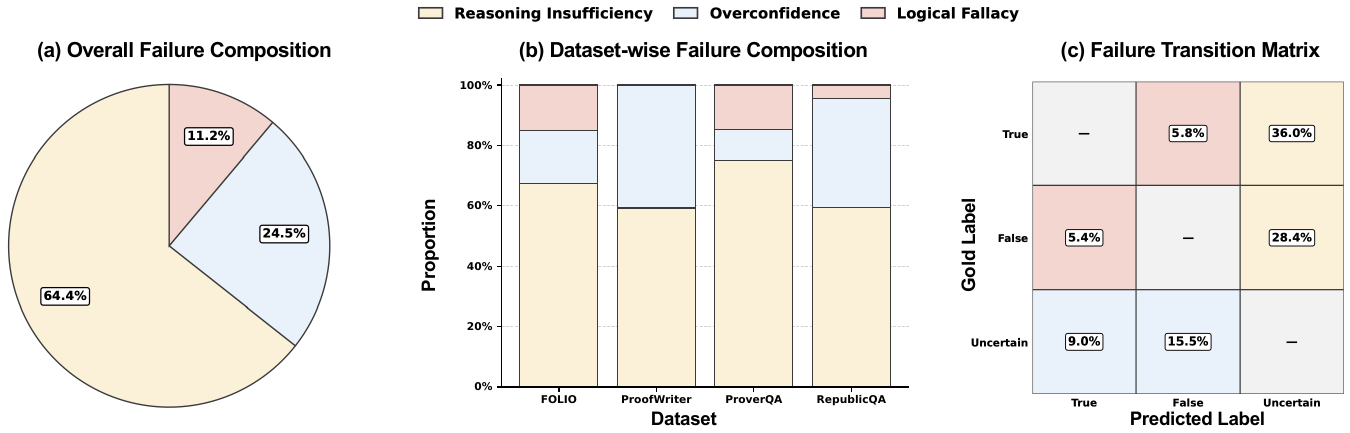}
    \caption{
    \textbf{Failure-mode analysis of HexLogicAgent.}
    Panel (a) shows the overall composition of erroneous predictions, categorized as reasoning insufficiency, overconfidence, and logical fallacy.
    Panel (b) reports the normalized failure composition for each benchmark, revealing that reasoning insufficiency dominates across datasets, while overconfidence is particularly prominent on ProofWriter and RepublicQA.
    Panel (c) presents the directional transitions from gold labels to predicted labels, with percentages computed over all erroneous predictions.
    }
    \label{fig:failure_mode_analysis}
\end{figure}

\subsubsection{Error Analysis}

We further analyze the erroneous predictions produced by HexLogicAgent with DeepSeek-V3.2, as summarized in Figure~\ref{fig:failure_mode_analysis}.
The analysis covers all incorrect predictions across the five benchmarks and groups them into three categories: reasoning insufficiency, overconfidence, and logical fallacy.

\textbf{(1) Reasoning insufficiency is the dominant failure mode.}
As shown in Figure~\ref{fig:failure_mode_analysis}(a), reasoning insufficiency accounts for $64.4\%$ of all errors.
These errors occur when the gold label is either \textit{True} or \textit{False}, but the model predicts \textit{Unknown}.
Overconfidence accounts for another $24.5\%$, whereas direct logical fallacies between \textit{True} and \textit{False} constitute only $11.2\%$.
Thus, most errors are associated with the treatment of uncertainty rather than direct reversals of logical polarity.

\textbf{(2) The distribution of failure modes varies across benchmarks.}
Figure~\ref{fig:failure_mode_analysis}(b) reports the normalized error composition on the four three-way classification benchmarks.
Reasoning insufficiency accounts for $67.4\%$ of the FOLIO errors, $59.3\%$ of the ProofWriter errors, $75.0\%$ of the ProverQA errors, and $59.5\%$ of the RepublicQA errors.
Overconfidence is particularly visible on ProofWriter and RepublicQA, where it represents $40.7\%$ and $36.0\%$ of the errors, respectively.
ProntoQA is omitted from this panel because it adopts a binary-label setting.

\textbf{(3) Errors are asymmetric with respect to uncertainty.}
The transition matrix in Figure~\ref{fig:failure_mode_analysis}(c) shows that the largest individual transition is $\textit{True}\rightarrow\textit{Unknown}$, accounting for $36.0\%$ of all errors, followed by $\textit{False}\rightarrow\textit{Unknown}$ at $28.4\%$.
In the opposite direction, $\textit{Unknown}\rightarrow\textit{False}$ and $\textit{Unknown}\rightarrow\textit{True}$ account for $15.5\%$ and $9.0\%$, respectively.
Direct polarity reversals are comparatively balanced and infrequent: $\textit{True}\rightarrow\textit{False}$ accounts for $5.8\%$, while $\textit{False}\rightarrow\textit{True}$ accounts for $5.4\%$.
These results identify uncertainty calibration as the main remaining source of error in HexLogicAgent.

\begin{figure}[width=\textwidth,pos=h!]
    \centering
    \includegraphics[
        width=0.95\textwidth,
        height=0.4\textheight,
        keepaspectratio
    ]{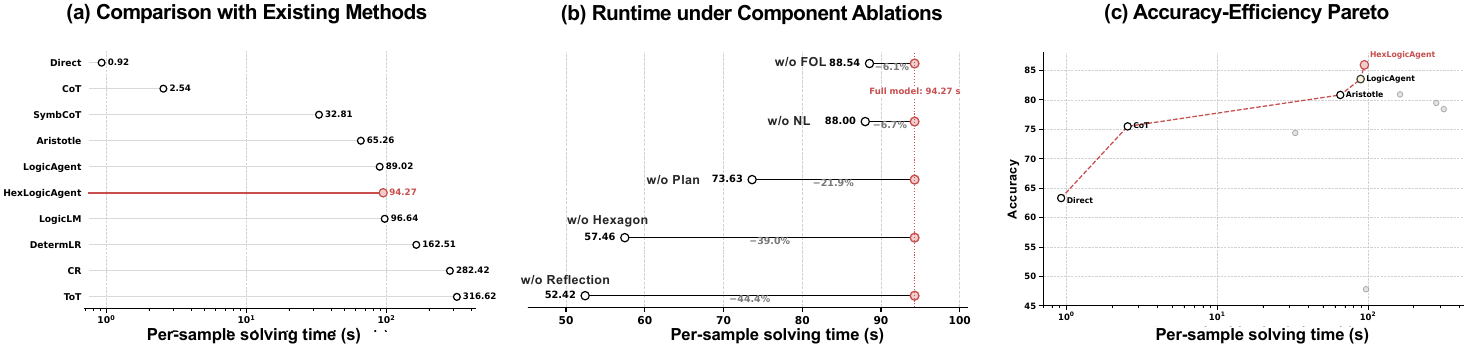}
    \caption{
    \textbf{Time-efficiency analysis of HexLogicAgent.}
    Panel (a) compares the per-sample solving time of HexLogicAgent and existing reasoning methods on a logarithmic scale.
    Panel (b) reports the per-sample solving time under different component ablations, with percentages denoting runtime changes relative to the full model.
    Panel (c) presents the accuracy--efficiency Pareto diagram on Qwen3-30B-A3B, where the dashed red line indicates the Pareto frontier.
    }
    \label{fig:time_efficiency_analysis}
\end{figure}

\subsubsection{Efficiency Analysis}
\label{sec:time-efficiency}

We evaluate the computational efficiency of HexLogicAgent through comparisons with existing reasoning methods, component-level runtime ablations, and accuracy--efficiency analysis, as shown in Figure~\ref{fig:time_efficiency_analysis}.

\textbf{(1) HexLogicAgent has a moderate runtime among structured reasoning methods.}
Figure~\ref{fig:time_efficiency_analysis}(a) reports the average per-sample solving time on a logarithmic scale.
HexLogicAgent requires $94.27$ seconds per sample, compared with $89.02$ seconds for LogicAgent and $96.64$ seconds for LogicLM.
It is slower than lightweight prompting methods such as Direct ($0.92$ seconds), CoT ($2.54$ seconds), SymbCoT ($32.81$ seconds), and Aristotle ($65.26$ seconds), but faster than DetermLR ($162.51$ seconds), CR ($282.42$ seconds), and ToT ($316.62$ seconds).

\textbf{(2) Reflective verification and hexagonal structuring account for most of the additional runtime.}
As shown in Figure~\ref{fig:time_efficiency_analysis}(b), removing reflective verification reduces the per-sample solving time from $94.27$ to $52.42$ seconds, corresponding to a $44.4\%$ reduction.
Removing the hexagonal structuring stage reduces the runtime to $57.46$ seconds ($39.0\%$), while removing planning reduces it to $73.63$ seconds ($21.9\%$).
By comparison, removing natural-language reasoning or formal-logic reasoning results in smaller reductions, producing runtimes of $88.00$ and $88.54$ seconds, respectively.
This decomposition shows that most of the computational overhead comes from constructing and verifying multiple semantic positions rather than from either reasoning representation alone.

\textbf{(3) HexLogicAgent remains on the accuracy--efficiency Pareto frontier.}
Figure~\ref{fig:time_efficiency_analysis}(c) compares average accuracy and per-sample solving time on Qwen3-30B-A3B.
The Pareto frontier consists of Direct, CoT, Aristotle, LogicAgent, and HexLogicAgent.
Compared with LogicAgent, HexLogicAgent improves the average accuracy from $83.55\%$ to $85.96\%$, while increasing the per-sample runtime from $89.02$ to $94.27$ seconds.
This corresponds to a $2.41$-point accuracy gain with an additional $5.25$ seconds per sample, indicating that the hexagonal extension provides an improved accuracy--efficiency trade-off over the semiotic-square baseline.

\textbf{Overall, the results answer RQ3 affirmatively.}
HexLogicAgent is more robust under increasing logical complexity, with larger gains on challenging reasoning tasks. Remaining errors are dominated by uncertainty calibration rather than logical polarity reversal, while the additional runtime mainly comes from hexagon construction and reflective verification. Despite this overhead, HexLogicAgent remains on the Pareto frontier, achieving a favorable accuracy--efficiency trade-off.

\section{Related Work}
\label{sec:related_work}

\textbf{LLM-Based Logical Reasoning.}
Existing LLM~\citep{li2024coupled,LoRA-Mixer-2026-ICLR,wu2026hotcomment} reasoning methods can be broadly categorized into three paradigms: Linear Reasoning (LR), Aggregative Reasoning (AR), and Symbolic Reasoning (SR). LR methods improve reasoning by generating a single chain of intermediate natural-language steps, including Chain-of-Thought (CoT)~\citep{COT-2022chain}, self-consistency~\citep{COT-SC-2022self}, and subsequent prompting strategies~\citep{Prompt-Design-2025prompt,GSM-apple-2024gsm,Efficient-reasoning-survey-2025NUS,yang2025magic,feng2025can,shen2026-zju}. AR methods further enhance robustness by exploring or aggregating multiple reasoning trajectories, such as Tree-of-Thought~\citep{TOT-2023tree}, CLOVER~\citep{CLOVER-ICLR-2024divide}, DetermLR~\citep{Determlr-ACL-2024-RenDaGaoling}, cumulative reasoning~\citep{Cumulative-reasoning-2023-TsinghuIIIS}, and other planning-based frameworks~\citep{Planning-in-logical-reasoning2023explicit,multi-step-2024exploring}. SR methods integrate LLMs with formal representations, symbolic deduction, or theorem proving, including NL-to-FOL translation~\citep{NL2FOL-2023harnessing,hahn2022formal}, LogicLM~\citep{LogicLM-2023}, SymbCoT~\citep{Symbolic-COT-2024faithful}, and Aristotle~\citep{Baseline-Aristotle2024xu}. Despite their differences, LR, AR, and SR all primarily improve logical deduction after proposition meanings have been interpreted, providing limited support for explicitly organizing semantic relations before reasoning.

\textbf{Semantic Organization for Logical Reasoning.}
Natural-language reasoning requires not only valid logical deduction but also appropriate semantic interpretation of propositions before deduction begins~\citep{Ambiguity-1993ambiguity,Logic-semantic-2007-Rutgers}. Formal studies in semiotics and logic have long investigated semantic organization through structured relations among propositions. Greimas' semiotic square~\citep{Greimas1982semiotics,Greimas-meaning1987,Greimas-semiotics1988maupassant} characterizes semantic opposition using contrary, contradictory, and subaltern relations, while the logical hexagon~\citep{hexagon-2012,logical-hexagon-2012,square2hexagon-2012} further introduces existential and mixed-state relations to provide a more complete semantic structure. However, these semantic theories have rarely been incorporated into modern LLM reasoning systems. Our previous LogicAgent framework~\citep{zhang2026semanticawarelogicalreasoningsemiotic} introduced semiotic-square-guided reasoning, whereas this work extends it to the logical hexagon, enabling richer semantic organization and more reliable logical reasoning through hexagon-guided structuring and reflective verification.

\textbf{Logical Reasoning Benchmarks.}
Logical reasoning benchmarks have evolved from synthetic rule-based datasets to increasingly realistic natural-language reasoning tasks. Early benchmarks, such as PrOntoQA~\citep{Benchmark-Pronto-2022language} and ProofWriter~\citep{Benchmark-ProofWriter-2020proofwriter}, emphasize deductive reasoning in controlled environments with formally defined rules. Subsequent benchmarks, including FOLIO~\citep{Benchmark-FOLIO2022folio} and ProverQA~\citep{Benchmark-ProverQA-ICLR-2025large}, introduce richer natural-language scenarios and first-order logical representations. More recently, RepublicQA~\citep{zhang2026semanticawarelogicalreasoningsemiotic} further improves realism by incorporating semantically diverse, context-dependent, and abstract propositions, providing a more challenging evaluation of LLM reasoning. These benchmarks have substantially advanced the evaluation of logical reasoning. Our work is complementary to this line of research: rather than constructing new benchmarks, we investigate how richer semantic organization before deduction can improve reasoning performance across existing benchmarks.

\section{Conclusion}
\label{sec:conclusion}

In this work, we investigated semantic-aware logical reasoning for large language models. We presented HexLogicAgent, a reasoning framework that extends the semiotic square to the logical hexagon by introducing existential and mixed-state semantics, enabling richer semantic organization before logical deduction. HexLogicAgent combines hexagon-guided semantic structuring with reflective verification to bridge semantic interpretation and logical reasoning within a unified framework. Extensive experiments on five logical reasoning benchmarks and multiple LLM backbones demonstrate consistent improvements over representative linear, aggregative, symbolic, and semiotic reasoning methods. Furthermore, our analyses show that HexLogicAgent improves semantic understanding, delays logical phase transitions under increasing logical complexity, and achieves a favorable accuracy--efficiency trade-off.

Looking forward, an important direction is to integrate semantic organization with the emerging generation of thinking LLMs. Although recent reasoning models incorporate long-chain deliberation and self-reflection, their exploratory reasoning often incurs substantial token consumption and reasoning latency. We believe that explicitly organizing proposition meanings before deduction can provide more structured reasoning guidance, reducing unnecessary exploration while improving reasoning reliability and efficiency. More broadly, we hope this work encourages further integration of formal semantic theories with LLM reasoning, leading to more reliable, interpretable, and efficient reasoning systems.

\printcredits


\section*{Acknowledgments}
This work is supported by the National Natural Science Foundation of China (Numbers 62272184 and 62402189), 
the China Postdoctoral Science Foundation (Numbers 2024M751012, 2025T180429, and GZC20230894),  
the Postdoctor Project of Hubei Province (Number 2024HBBHCXB014),
the Natural Science Foundation of Hubei Province No.JCZRMS202600758,
and Sponsored by CIPS-SMP-Zhipu Large Model Fund (CIPS-SMP20250306),
The computation is completed in the HPC Platform of Huazhong University of Science and Technology.

\clearpage

\appendix
\section*{Appendix: Case Study}
\label{app:case-study}

We use Socrates' counterexample to Cephalus' definition of justice in Plato's
\textit{Republic} to illustrate how HexLogicAgent distinguishes a local
counterexample from a universal negation.

\begin{figure}[width=\textwidth,pos=h!]
\centering

\begin{casebox}
\small
\setlength{\parindent}{0pt}
\setlength{\parskip}{1.5pt}
\renewcommand{\arraystretch}{0.96}

\casetitle{Context and Question}

Returning borrowed money to its rightful owner under normal circumstances is
an act of repaying a debt. Returning borrowed money restores rightful property
and does not endanger innocent people. Any action satisfying these conditions
is just. Returning a borrowed weapon is also a debt repayment. If its owner has
become mentally unstable, returning the weapon enables harm to innocent people.
Any action enabling harm to innocent people is not just.

\begin{tabularx}{\linewidth}{
    >{\bfseries}p{0.13\linewidth}
    X
}
Question &
Is the statement ``All acts of repaying debts are just'' correct? \\
\end{tabularx}

\casetitle{Stage 1: Hexagonal Semantic Structuring}

HexLogicAgent first normalizes the target proposition into the subject--property
tuple
\(\mathcal{T}(Q)=\langle x,\;DebtRepayment(x),\;Just(x)\rangle\)
and identifies its semantic anchor.

\begin{tabularx}{\linewidth}{
    >{\centering\arraybackslash\bfseries}p{0.08\linewidth}
    X
}
\toprule
Position & Logical form \\
\midrule
\(A\) &
\(\forall x\,(DebtRepayment(x)\rightarrow Just(x))\) \\

\(E\) &
\(\forall x\,(DebtRepayment(x)\rightarrow \neg Just(x))\) \\

\(I\) &
\(\exists x\,(DebtRepayment(x)\land Just(x))\) \\

\(O\) &
\(\exists x\,(DebtRepayment(x)\land \neg Just(x))\) \\

\(U\) &
\(U=A\lor E\) \\

\(Y\) &
\(Y=I\land O\) \\
\bottomrule
\end{tabularx}

The target statement is anchored at \(\rho(Q)=A\), whose contradictory
position is \(\neg\rho(Q)=O\).

\casetitle{Stage 2: Logical Reasoning}

\textbf{Premises.}
Let \(r_m\) denote the ordinary money-return action and \(r_w\) denote the
weapon-return action.

\begin{enumerate}[
    leftmargin=1.6em,
    itemsep=0pt,
    topsep=1pt,
    parsep=0pt,
    partopsep=0pt
]
    \item \(DebtRepayment(r_m)\).
    \item \(RestoresRightfulProperty(r_m)\) and
          \(\neg EndangersInnocentPeople(r_m)\).
    \item \(\forall x((RestoresRightfulProperty(x)\land
          \neg EndangersInnocentPeople(x))\rightarrow Just(x))\).
    \item \(DebtRepayment(r_w)\) and \(MentallyUnstableOwner(r_w)\).
    \item \(\forall x(MentallyUnstableOwner(x)\rightarrow
          EnablesHarmToInnocentPeople(x))\).
    \item \(\forall x(EnablesHarmToInnocentPeople(x)\rightarrow
          \neg Just(x))\).
\end{enumerate}

\begin{tabularx}{\linewidth}{
    >{\bfseries}p{0.18\linewidth}
    X
}
Ordinary repayment &
The premises derive \(Just(r_m)\). Since \(DebtRepayment(r_m)\) also holds,
the existential affirmative position is supported:
\(I=\mathrm{True}\). \\[1pt]

Weapon repayment &
The owner's mental instability entails
\(EnablesHarmToInnocentPeople(r_w)\), which yields
\(\neg Just(r_w)\). Together with \(DebtRepayment(r_w)\), the existential
negative position is supported:
\(O=\mathrm{True}\). \\
\end{tabularx}

\casetitle{Stage 3: Hexagon-Guided Reflective Verification}

Since both existential propositions are supported,
\(I=\mathrm{True}\) and \(O=\mathrm{True}\),
the verifier derives the mixed state
\(Y=I\land O=\mathrm{True}\),
indicating that debt repayments form a mixed rather than uniform class.

The target proposition is anchored at the universal affirmative position
\(A\). According to the contrary constraint
\(\mathcal{C}_{\mathrm{ctr}}=\{A,E,Y\}\),
the mixed state rejects the universal affirmative, yielding
\(A=\mathrm{False}\). Meanwhile, because \(I=\mathrm{True}\),
the universal negative proposition cannot hold, and thus
\(E=\mathrm{False}\).

Therefore, the weapon-return example establishes only the local
counterexample \(O=\neg A\), rather than the stronger global negative
position \(E\). Accordingly, the final verdict is

\begin{center}
\fbox{\textbf{False}}
\end{center}

\end{casebox}

\caption{Case study illustrating how HexLogicAgent distinguishes a local counterexample from a universal negation. }
\label{fig:compact_case}

\end{figure}

\bibliographystyle{unsrtnat}

\bibliography{main}



\end{document}